\documentclass[11pt]{article} 
\usepackage{fullpage}

\usepackage[utf8]{inputenc} 
\usepackage[T1]{fontenc}    
\usepackage{url}            
\usepackage{booktabs}       
\usepackage{amsfonts}       
\usepackage{nicefrac}       
\usepackage{microtype}      
\usepackage{thmtools}
\usepackage{thm-restate}

\usepackage{times}
\usepackage{tikz}
\usepackage{mathtools}
\usepackage{amsthm}
\usepackage{amssymb}
\usepackage{bbm}
\usepackage{verbatim}
\usepackage{enumitem}
\usepackage{xspace}
\usepackage{todonotes}
\usepackage{amsfonts,mathrsfs,color}
\usepackage{xcolor}
\usepackage{graphicx}
\usepackage{booktabs}
\usepackage{nicefrac}
\usepackage{mdframed}
\usepackage{contour}
\contourlength{0.1pt}
\contournumber{10}
\usepackage{natbib}
\usepackage{multirow}
\usepackage[ruled,linesnumbered,noend]{algorithm2e} 
\SetKwComment{Comment}{/* }{ */}
\SetKwProg{Fn}{function}
\SetAlFnt{\relax}
\SetAlCapFnt{\relax}
\SetAlCapNameFnt{\relax}
\SetAlCapHSkip{0pt}
\usepackage{tikz}
\usepackage{xcolor}

\newcommand{\declarecolor}[2]{\definecolor{#1}{RGB}{#2}\expandafter\newcommand\csname #1\endcsname[1]{\textcolor{#1}{##1}}}
\declarecolor{White}{255, 255, 255}
\declarecolor{Black}{0, 0, 0}
\declarecolor{Maroon}{128, 0, 0}
\declarecolor{Coral}{255, 127, 80}
\declarecolor{Red}{182, 21, 21}
\declarecolor{LimeGreen}{50, 205, 50}
\declarecolor{DarkGreen}{0, 80, 0}
\declarecolor{Purple}{146, 42, 158}
\declarecolor{Navy}{0, 0, 128}
\declarecolor{LightBlue}{84, 101, 202}
\usepackage{hyperref}
\hypersetup{
    colorlinks,
    citecolor=DarkGreen,
    linkcolor=LightBlue}
\usepackage[noabbrev, capitalise, nameinlink]{cleveref}
\usepackage{xcolor}

\newtheorem{definition}{Definition}
\newtheorem{theorem}{Theorem}
\newtheorem*{theorem*}{Theorem}
\newtheorem{lemma}{Lemma}

\newtheorem{corollary}{Corollary}

\newtheorem{proposition}{Proposition}
\newtheorem{remark}{Remark}
\newtheorem{property}{Property}

\newtheorem*{informaltheorem*}{Informal Theorem}

\DeclareMathOperator{\Unif}{Unif}
\newcommand{\reg}{\mathrm{Reg}}

\usepackage{pifont}

\DeclareMathOperator*{\argmin}{argmin}
\DeclareMathOperator*{\argmax}{argmax}

\def\+#1{\mathcal{#1}}
\def\-#1{\mathbb{#1}}

\newcommand{\notshow}[1]{{}}
\newcommand{\AutoAdjust}[3]{{\mathchoice{ \left #1 #2  \right #3}{#1 #2 #3}{#1 #2 #3}{#1 #2 #3}}}
\newcommand{\Xcomment}[1]{{}}

\newcommand{\InParentheses}[1]{\AutoAdjust{(}{#1}{)}}
\newcommand{\InBrackets}[1]{\AutoAdjust{[}{#1}{]}}

\renewcommand{\part}[2]{\frac{\partial #1}{\partial #2}}

\newcommand{\R}{\mathbbm{R}} 
\newcommand{\N}{\mathcal{N}}

\newcommand{\X}{\mathcal{X}}
\newcommand{\Y}{\mathcal{Y}}

\usepackage{color-edits} 
\addauthor[Yang Cai]{yc}{cyan}

\newcommand{\ualign}{{U-alignment}\xspace}

\title{Asymptotic Universal Alignment:\\ A New Alignment Framework via Test-Time Scaling}

\author{
Yang Cai\thanks{Authors are alphabetically ordered.}\\
Yale University\\
\texttt{yang.cai@yale.edu}\\
\and 
Weiqiang Zheng\footnotemark[1]\\
Yale University\\
\texttt{weiqiang.zheng@yale.edu}\\
}

\begin{document}
\maketitle
\begin{abstract}
Aligning large language models (LLMs) to serve users with heterogeneous and potentially conflicting preferences is a central challenge for personalized and trustworthy AI. We formalize an ideal notion of \emph{universal alignment} through \emph{test-time scaling}: for each prompt, the model produces $k\ge 1$ candidate responses and a user selects their preferred one. We introduce \emph{$(k,f(k))$-robust alignment}, which requires the $k$-output model to have win rate $f(k)$ against any other single-output model, and \emph{asymptotic universal alignment (U-alignment)}, which requires $f(k)\to 1$ as $k\to\infty$. Our main result characterizes the optimal convergence rate: there exists a family of \emph{single-output} policies whose $k$-sample product policies achieve U-alignment at rate $f(k)=\frac{k}{k+1}$, and no method can achieve a faster rate in general. 

We show that popular post-training methods, including Nash learning from human feedback (NLHF), can fundamentally underutilize the benefits of test-time scaling. Even though NLHF is optimal for $k=1$, sampling from the resulting (often deterministic) policy cannot guarantee win rates above $\tfrac{1}{2}$ except for an arbitrarily small slack. This stems from a lack of output diversity: existing alignment methods can collapse to a single majority-preferred response, making additional samples redundant. In contrast, our approach preserves output diversity and achieves the optimal test-time scaling rate. In particular, we propose a family of symmetric \emph{multi-player alignment games} and prove that any symmetric Nash equilibrium policy of the $(k+1)$-player alignment game achieves the optimal $(k,\frac{k}{k+1})$-robust alignment. Finally, we provide theoretical convergence guarantees for self-play learning dynamics in these games and extend the framework to opponents that also generate multiple responses.
\end{abstract}
\thispagestyle{empty}
\setcounter{page}{0}
\newpage

\section{Introduction}
Large language models (LLMs) have demonstrated remarkable versatility across domains such as text generation, code synthesis, information retrieval, and mathematical reasoning. Owing to their broad generalization capabilities, LLMs are increasingly integrated into users’ daily workflows to assist with information seeking, solution generation, and decision-making. However, users vary widely in their preferences---when given the same prompt, they may favor responses that differ in content, tone, style, values, or level of detail. Consequently, aligning LLMs to serve people with diverse values and perspectives~\citep{sorensen2024position} has been recognized as a central challenge in building more personalized and trustworthy AI systems.

We formalize an ideal notion of alignment across diverse preferences, in which an AI system is aligned with every possible user—a property we term \emph{universal alignment}. 
While one could, in principle, achieve universal alignment by training or fine-tuning a bespoke model for each individual user, such an approach is prohibitively expensive given the substantial computational and human resources required. 
In this paper, we ask to what extent universal alignment can be approached using a \emph{single model}. 
In particular,
\begin{align*}
    &\textit{Can we efficiently approach \textbf{universal alignment }} \textit{using a single model and test-time scaling?}
\end{align*}
At first glance, this may seem impossible: a single model must reconcile conflicting preferences, and satisfying one group may inevitably come at the expense of another. 
We show that, with appropriate post-training alignment and test-time scaling, such reconciliation is, in fact, possible. 
Moreover, we provide theoretical guarantees characterizing how the amount of test-time scaling required relates to the desired level of alignment.

\subsection{Our Model and Results}
In this section, we introduce a formal framework for studying universal alignment under test-time scaling 
and establish theoretical guarantees and fundamental limits within this framework. 
To facilitate the discussion, we assume that each user has an individual ranking over responses $y$ for every prompt $x$. 
This simplified model captures the essential ideas while keeping the presentation clear; 
however, our results extend to other preference structures such as Plackett--Luce and mixtures of Plackett--Luce 
(see \cref{sec:prelim} for details). 
Intuitively, we allow the model to generate $k \ge 1$ candidate responses for each prompt 
and let the user select their preferred one.\footnote{%
For clarity, we assume that the user selects their favorite response. 
In practice, the model provider may employ a lightweight auxiliary model to predict or select the user’s preferred response 
based on their features and past interactions.}
\subsubsection{Robust Alignment and Universal Alignment}
We begin by defining the notion of \emph{robust alignment}.

\begin{definition}[$(k,f(k))$-Robust Alignment]
Let $f: \mathbb{N}_{>0} \to [0,1]$ be a  function, and let $k \in \mathbb{N}_{>0}$ denote the number of responses generated at test time. 
A policy $\pi: \+X \to \Delta(\+Y^k)$ is said to achieve $(k,f(k))$-robust alignment if, for every prompt $x \in \+X$, its win rate against any other policy $\pi': \+X \to \Delta(\+Y)$ satisfies $\Pr[\pi \succcurlyeq \pi' \mid x] \;\ge\; f(k)$~\footnote{We formally define the probability $\Pr[\pi \succcurlyeq \pi' \mid x]$ in~\Cref{sec:prelim}. 
For now, it can be interpreted as the probability that a randomly selected user from the population 
prefers at least one of the $k$ responses generated by $\pi$ to the response generated by $\pi'$.}.
\end{definition}
An immediate implication of $(k, f(k))$-robust alignment is that, for any prompt $x$, a robustly aligned policy $\pi$ must include every response $y$ that is the favorite of more than a $(1 - f(k))$-fraction of the population. Otherwise, if such a response $y$ is omitted by $\pi$, a policy $\pi'$ that always outputs $y$ on input $x$ would violate the $(k, f(k))$-robust alignment condition. Equipped with the definition of robust alignment, we now introduce the notion of \emph{asymptotic universal alignment}, which we may refer to simply as U-alignment throughout the paper.

\begin{definition}[Asymptotic Universal Alignment (U-Alignment)]
\label{def:ualignment}
Let $f: \mathbb{N}_{>0} \to [0,1]$ be a rate function satisfying $\lim_{k \to \infty} f(k) = 1$. 
We say that a family of policies $\{\pi_k\}_{k \in \mathbb{N}_{>0}}$ achieves \emph{asymptotic universal alignment} with rate $f$ 
if, for every $k \in \mathbb{N}_{>0}$, the policy $\pi_k$ achieves $(k, f(k))$\nobreakdash-robust alignment. 
In other words, as the number of generated responses $k$ increases, the alignment rate $f(k)$ approaches~$1$, 
corresponding to alignment with almost all users in the population.
\end{definition}

 Note that any policy $\pi$ that assigns nonzero probability to every possible response under any prompt---e.g., one that samples a response uniformly at random---technically satisfies our notion of \ualign by sampling $k$ times independently from $\pi$ for each prompt.  However, the convergence rate in this case is extremely slow: unless $k$ is on the order of $|\Y|$, we have $f(k) = o(1)$. 
Such a convergence rate is clearly undesirable, both theoretically and practically. 
Can we approach universal alignment more efficiently?

\subsubsection{Our Result}
Our main result provides a positive answer to this question by characterizing the optimal convergence rate.

\begin{informaltheorem*}
There exists a family of single-output policies $\{\pi_k\}_{k \in \mathbb{N}_{>0}}$ such that the corresponding test-time scaled policies $\{\pi_k^{\otimes k}\}_{k \in \mathbb{N}_{>0}}$ achieve \ualign with rate $\nicefrac{k}{k+1}$. 
Here, each $\pi_k$ maps any prompt in $\X$ to a distribution over $\Y$, and $\pi_k^{\otimes k}$ denotes the policy that independently samples from $\pi_k$ $k$ times for each prompt. 
Moreover, no family of policies can achieve \ualign at a faster rate.
\end{informaltheorem*}

This statement provides the theoretical foundation for \ualign by characterizing the optimal convergence rate achievable through test-time scaling. The formal version of this informal theorem is given in \Cref{theorem:main}. In the remainder of this section, we examine the features, extensions, and implications of this result from four complementary perspectives. 

We begin by highlighting the importance of achieving this optimal rate using single-output policies, which are more consistent with current model architectures and training–inference pipelines than multi-output policies. Next, we address computational considerations, showing how optimal policies can be characterized via self-play in symmetric multi-player alignment games. We then analyze the limitations of existing post-training alignment methods under test-time scaling. Finally, we explore the broader connection between \ualign and diversity preservation in post-training alignment.

\paragraph{Optimal \ualign via Single-Output Policies.}
By drawing a connection to the notion of the Condorcet winning set in social choice theory, we can leverage existing results to establish the existence of a family of policies that achieve \ualign 
at rate $\nicefrac{k}{k+1}$ for any $k \ge 1$ (\Cref{theorem:multioutput alignment}). 
However, this existence result applies only to a \emph{multi-output} policy that maps a prompt $x \in \X$ directly to a distribution over $\Y^k$, 
which is less desirable than achieving the same guarantee by scaling a single-output policy at test time. 
In particular:
\begin{itemize}[leftmargin=*]
    \item training a multi-output policy is nonstandard and requires optimization over a response space whose size grows exponentially with $k$; and
    \item inference from a multi-output policy is slower, since parallel computation cannot be exploited for acceleration.    
\end{itemize}

\noindent Our main result shows that, for any $k \in \mathbb{N}_{>0}$, a single-output policy suffices to achieve the same optimal convergence rate, thereby avoiding these drawbacks.

\paragraph{Computing the Optimal \ualign Policies via Self-Play.}
Our main result establishes the existence of a family of single-output policies that achieve the optimal convergence rate to \ualign. 
We next characterize this family by relating it to a class of symmetric $(k+1)$-player alignment games defined in~\Cref{dfn:multi-player alignment game}. 
For any $k \in \mathbb{N}_{>0}$, we show that the symmetric Nash equilibrium policy of such a game corresponds to $\pi_k$ in our main result (see~\Cref{theorem:single-output good}). This formulation enables self-play methods for approaching optimal \ualign. We provide convergence guarantees for self-play learning dynamics in \Cref{proposition:fixed point} and \Cref{prop:self-play}. 


\paragraph{Limits of Test-Time Scaling of NLHF and RLHF.} 
We briefly review two popular alignment approaches and discuss their limitations under test-time scaling. Formal definitions can be found in \Cref{sec:pre RLHF NLHF}.
\begin{itemize}[leftmargin=*]
    \item \textbf{RLHF.}\; Reinforcement learning from human feedback (RLHF)~\citep{christiano2017deep, bai2022training, ouyang2022training} 
    assumes that aggregated human preferences follow the Bradley--Terry (BT) model~\citep{bradley1952rank}. 
    RLHF first trains a scalar reward model from human preference data, and then optimizes the base model with respect to this learned reward via reinforcement learning. 
    However, any scalar reward model fails to capture the full diversity of human preferences—particularly when preferences are non-transitive---introducing systematic bias and potential misalignment in RLHF~\citep{munos2024nash}. 

    \item \textbf{NLHF.}\; Nash learning from human feedback (NLHF)~\citep{munos2024nash, swamy2024minimaximalist} 
    relaxes the scalar reward assumption and directly models general preferences. 
    NLHF formulates alignment as a two-player zero-sum game and seeks its Nash equilibrium policy.
    The resulting Nash policy satisfies a desirable guarantee: it achieves at least a 50\% win rate against any other policy, 
    corresponding to $(1, \tfrac{1}{2})$-robust alignment in our framework. 
\end{itemize}

In summary, RLHF builds on the BT model and fails to capture diverse preferences, whereas NLHF achieves the optimal robust-alignment guarantee for $k=1$. Given NLHF's prevalence, one may wonder whether test-time scaling NLHF yields nontrivial convergence to \ualign.
We provide a negative answer in \Cref{prop:limit of NLHF}: for any $k\ge 1$, there exist instances in which test-time scaling of NLHF cannot guarantee $(k,\tfrac{1}{2}+\varepsilon)$-robust alignment beyond an arbitrarily small slack. 
The intuition is that NLHF can lack response diversity and may collapse to (nearly) deterministic policies, which makes additional samples redundant. We elaborate on this point in the next paragraph.

\paragraph{U-alignment and Diversity Preservation during Post-Training.}
A necessary condition for \ualign is the ability to generate a diverse set of responses that reflect heterogeneous human preferences. 
However, post-training alignment methods such as RLHF and NLHF often induce \emph{mode collapse}, reducing output diversity~\citep{kirk2024understanding}. 
When a majority of users prefer a particular response, RLHF and NLHF tend to converge toward producing it almost deterministically. 
Although aligning with the majority’s preference may seem optimal, disregarding minority preferences prevents RLHF and NLHF from achieving \ualign. 
The resulting lack of output diversity also limits the effectiveness of LLMs in applications that require creativity and variation, such as creative writing and synthetic data generation. 

Our work demonstrates that it is possible to post-train a model to improve alignment while preserving diversity, 
challenging the common view that post-training alignment inevitably leads to mode collapse and reduced diversity.

\paragraph{Multi-output Opponents.} The notion of robust alignment and \ualign can be naturally generalized to the setting where we compete with an opponent that can also outputs $\ell \ge 1$ responses. We define two notions of $\ell$-\ualign: the first one is against an arbitrary opponent, while the second one is against a \emph{product opponent} that can only sample $\ell$ responses from a fixed distribution, i.e., test-time scaling of a fixed policy. We show that the optimal \ualign policy (the case $\ell = 1$) also achieves non-trivial performance for $\ell$-\ualign when $\ell \ge 2$ (\Cref{thm:arbitrary-multi-output-opponent}). For product opponents, we consider a broader class of random utility preferences that generalizes the PL and ranking models (\Cref{dfn:rum-preference}), and establish a tight $\nicefrac{k}{k+\ell}$ rate for the optimal $\ell$-\ualign (\Cref{thm:sharp-product-opponent}). Moreover, we show that for product opponent, the optimal policy admits a two-player zero-sum formulation with  a convex-concave objective, which enables self-player learning dynamics converging to optimal policy with last-iterate convergence (\Cref{corollary:last-iterate}). The last-iterate convergence result generalizes the recent results for NLHF where $\ell = 1$~\citep{liu2026comal,wang2025magnetic}.

\subsection{Related Works}
\label{sec:related works}
\paragraph{Social Choice Theory} Social choice theory focuses on the problem of choosing a winner or a committee from a set of alternatives according to voters' preferences over the alternatives, which finds application in LLMs recently (see e.g., \citep{fish2024generative,shi2025fundamental,xiao2025theoretical}). We review the notions in social choice theory that are closely related to our work.  A well-known impossibility result is Condorcet's paradox. This paradox shows that a \emph{Condorcet winner}, an alternative that is preferred by at least half of the voters against any other alternative, may not exist. However, if we allow randomization, a \emph{maximal lottery}~\citep{kreweras1965aggregation,fishburn1984probabilistic} exists,  which is a distribution over alternatives and guarantees $\nicefrac{1}{2}$ expected win rate against any other alternative. Maximal lottery is the foundation for Nash learning from human feedback (NLHF)~\citep{munos2024nash,swamy2024minimaximalist}. Our work focuses on the test-time scaling version of NLHF and is related to another relaxed notion of Condorcet winner called \emph{Condorcet winning set}~\citep{elkind2011choosing,elkind2015condorcet}, a committee of alternatives such that there is no other alternative that has a win rate $> \nicefrac{1}{2}$ against all alternatives in the committee. A generalized notion of Condorcet winning set is \emph{$\alpha$-undominated committee}. A committee $S$ is $\alpha$-undominated if for all $a \notin S$, less than $\alpha$-fraction of voters prefer $a$ over each member of $S$. There is a line of work on finding the smallest $\alpha$-undominated set~\citep{cheng2020group, jiang2020approximately, charikar2025six,nguyen2025few, charikar2025metric}. Our alignment framework is closely related to the $\alpha$-undominated set, as achieving $(k,f(k))$-robust alignment is equivalent to finding a randomized committee of size $k$ that is $(1-f(k))$-undominated. We show in \Cref{sec:multi-output} that we can adapt existing results from social choice to show the existence of a multi-output policy that achieves \ualign with rate $\nicefrac{k}{k+1}$. However, multi-output policies are unsatisfactory in the LLM setting (see \Cref{sec:limits of multi-output policy} for a detailed discussion). Achieving \ualign with test-time scaling of single-output policy requires new insights from game theory, as we presented in \Cref{sec:single-output}.

\paragraph{Game-Theoretic Approaches in Alignment} As we have discussed, the Nash learning from human feedback (NLHF)~\citep{munos2024nash,swamy2024minimaximalist} approach formulates the LLM alignment problem as a two-player zero-sum game and aims to find the Nash equilibrium (NE) policy. The NE policy coincides with the maximal lottery in social choice and achieves a $\nicefrac{1}{2}$ win rate against any other policies. There are many recent developments on NLHF (see e.g., \citep{calandriello2024human,wu2025selfplay,zhang2025iterative,maura2025jackpot,liu2025statistical}). In particular, initiated by \citet{liu2024comal,wang2025magnetic}, there is a line of works on last-iterate convergent methods for NLHF~\citep{zhou2025extragradient,wu2025multiplayer,tiapkin2025accelerating}, that applies recent advances in online learning and game theory~\citep{rakhlin2013optimization, daskalakis2018limit, wei2021linear, cai2022finite, cai2024fast} to LLM alignment.

\paragraph{Post-Training and Test-Time Scaling} One of our work's broader implications is that post-training and test-time scaling may not be aligned, and we need to carefully design the post-training objective in order to achieve good test-time performance. We discuss several related works on the misalignment between post-training and test-time scaling. The work by~\citet{yue2025does} shows that the model after post-training via reinforcement learning performs worse than the base model on pass@$K$ (the model is considered correct if at least one of the $K$ samples is correct) when $K$ is large, indicating that post-training may even hurt the reasoning capabilities of the model. Relatedly, \citet{chen2025rethinking} show that in math reasoning settings, supervised finetuning (SFT) with the cross entropy loss can be misaligned with the test-time scaling approach of pass@$K$. \citet{chen2025rethinking} then suggest a new SFT loss that is more aligned with pass@$K$. Our work focuses on post-training with general preferences, and we show that existing methods like RLHF and NLHF do not align well with test-time scaling (see \cref{prop:limit of NLHF} for a formal argument). Our new objective, U-alignment, leads to a provable test-time scaling guarantee.

\paragraph{Comparison with~\citep{wu2025multiplayer}}e A concurrent and independent work~\citep{wu2025multiplayer} that also propose a multiplayer-game formulation for LLM alignment.  However, the game objectives in their work and our \Cref{dfn:multi-player alignment game} are very different:
\begin{itemize}[leftmargin=*]
    \item In the game defined in \citep{wu2025multiplayer}, each player $i$'s utility is the \emph{averaged} win rate $\frac{1}{k}\sum_{j\ne i} \-P[{\pi_i \succ \pi_j}]$. This game is symmetric, and we note that the NLHF policy is a symmetric Nash equilibrium of the game. Therefore, solving this multi-player game does not offer a theoretical advantage over the two-player game in NLHF. Moreover, test-time scaling of the Nash equilibrium policy does not achieve \ualign (\Cref{prop:limit of NLHF}). 
    \item In our multi-player alignment game (\Cref{dfn:multi-player alignment game}), each player $i$'s utility is the win rate against the product policy of other players' policies $\-P[{\pi_i \succ \otimes_{j\ne i}\pi_j}]$. Our game is symmetric, and we prove that a symmetric Nash equilibrium achieves \ualign with optimal rate (\Cref{theorem:single-output good}).
\end{itemize}
We remark that~\citep{wu2025multiplayer} shows that with diverse opponents, training with their objective exhibits better empirical performance than NLHF. However, using different opponent models breaks the symmetry and no longer solves the original game. 

\section{Preliminaries}\label{sec:prelim}
Let the universe of prompts be $\+X$ and the universe of responses be $\+Y$, both assumed to be finite.
A single-output language model (policy) is a mapping $\pi: \+X \rightarrow \Delta(\+Y)$, and $\pi(y \mid x) \in [0,1]$ the  probability of outputting response $y \in \+Y$ given a prompt $x \in \+X$. 
A multi-output language policy $\pi: \+X \rightarrow \Delta(\+Y^k)$ specifies the probability $\pi(S\mid x)$ of outputting a multiset of responses $S \subseteq \+Y^k$ given a prompt $x \in \+X$.\footnote{Here we slightly abuse notations to denote the universe of multisets of size $k$ as $\+Y^k$, which usually denotes tuples of size $k$.} We sometimes write a multiset $S$ as a vector $m_S \in \-Z_{+}^{\+Y}$ where $m_S(y)$ is the number of copies of $y$ in $S$. Given a vector $m \in \-Z_{+}^{\+Y}$, we denote $S(m)$ the corresponding multiset.
Given a single-output policy $\pi$ and $k \ge 1$, the product policy $\pi^{\otimes k}$ is the $k$-output policy that outputs $k$ independent samples from $\pi$ for any given prompt. For two multisets $S$ and $S'$, we define $S+S'$ to be the multiset union of $S$ and $S'$, i.e., the multiplicity of each element in $S+S'$ is the sum of its multiplicities in $S$ and in $S'$. 

\subsection{Preferences} 
Our results hold for a general class of preferences. We first introduce the Plackett-Luce (PL) preference model~\citep{luce1959individual}, a generalization of the Bradley-Terry (BT) model~\citep{bradley1952rank} from comparing two responses to comparing two sets of responses, which is widely used in LLM alignment. We derive a simple and useful property of the aggregated population preference when each person has a PL preference. We then consider the standard social choice setting in which each person has a complete ranking of the responses. 

\paragraph{Mixtures of Plackett-Luce Models}
    There are $n$ different types of \emph{Plackett-Luce (PL)} preferences. Each preference $i \in [n]$ is associated with a reward function $r_i(x,y)$. For any pair of multisets of responses $S, S'$, the preference of $S$ over $S'$ is defined as
    \begin{align*}
        \-P_i[S \succcurlyeq S' \mid x] = \-P_i[S \succ S' \mid x] := \frac{\sum_{y \in S} \exp(r_i(x,y))}{\sum_{y \in S} \exp(r_i(x,y)) + \sum_{y\in S'} \exp(r_i(x,y))}, \forall x, S, S'.
    \end{align*}
    Given a distribution $\+D$ over $[n]$, the aggregated \emph{population preference} $\-P_{\+D}$ is a mixture of the $n$ PL preferences. 
    Specifically, for any $k, k'\ge 1$ and a $k$-output policy $\pi$ and  $k'$-output policy $\pi'$, the population preference of $\pi \succcurlyeq \pi'$ is defined as
    \begin{align*}
        \-P_{\+D}[\pi \succcurlyeq \pi'\mid x] = \-P_{\+D}[\pi \succ \pi'\mid x] := \-E_{i \sim \+D} \-E_{S \sim \pi(\cdot \mid x), S' \sim \pi'(\cdot \mid x)} \InBrackets{ \-P_i [S \succ S' \mid x]}, \forall x, \pi, \pi'.
    \end{align*}
    We identify and summarize useful properties that the population preference $\-P_{\+D}$ has. 
    \begin{property}\label{dfn:general preferences}
    A mixture of PL preferences $\-P_{\+D}$ satisfies the following properties
    \begin{itemize}[leftmargin=*]
        \item[(1)] \textnormal{Antisymmetry:} $\-P_{\+D}[\pi \succcurlyeq \pi'\mid x] + \-P_{\+D}[\pi' \succ \pi\mid x] = 1$ for any two multi-output policies $\pi, \pi'$ and prompt $x$. 
        \item[(2)] \textnormal{Multi v.s. Single Copy:} $\-P_{\+D}[\pi^{\otimes k} \succcurlyeq \pi \mid x] \ge 1 - \frac{1}{k+1}$ for any single-output policy $\pi$ and $x$.
        \item[(3)] \textnormal{Subadditivity:} $\-P_{\+D}[S_1 +S_2 \succ S \mid x] \le \-P_{\+D}[S_1 \succ S \mid x] + \-P_{\+D}[S_2 \succ S \mid x]$ for any 
        multisets $S_1, S_2, S$ and prompt $x$.
        \item[(4)] \textnormal{Submodularity:} For $m\in \-Z_{+}^{\+Y}$, let $F_{x,y}(m):=\-P_{\+D}[S(m) \succcurlyeq y \mid x]$. For any prompt $x$ and response $y$, the function $F_{x,y}(\cdot)$ satisfies $F_{x,y}(0) = 0$, is monotone and DR-submodular.\footnote{We provide definitions and proofs in \Cref{app:pre}.}
    \end{itemize}
    \end{property}
    \begin{remark}\label{remark:PL}
    All of our results hold for any population preference satisfying \Cref{dfn:general preferences}(1--3); \Cref{thm:fixed scaling} further assumes condition~(4). This covers a broad class of preference models. In particular, the following two generalizations of the PL model have \Cref{dfn:general preferences} by definition: (1) each user's preference is a mixture of PL preference, and the population preference is a distribution over the possible mixtures; (2) based on (1), we further allow the population distribution for different prompts $x \in \+X$ to be different. 
    \end{remark} 
\paragraph{Population of Rankings} Now, we consider the setting where each user has a total ranking over responses for each prompt.
We model each possible preference $\succ_i$ as orders $\{\succ_{i,x}\}_{x \in \+X}$ over $\+Y$ such that $y \succ_{i,x} y'$ means $y$ is preferred than $y'$ under prompt $x$. To incorporate ties between identical responses, we extend the preference to weak preferences $\succcurlyeq_i$ and define the following function
\begin{align*}
    &\boldsymbol{1}[y \succcurlyeq_{i,x} y'] := \begin{cases}
        1, \quad y \succ_{i,x} y' \text{ or } y=y'\\
        0, \quad y' \succ_{i,x} y
    \end{cases},\\
    &\boldsymbol{1}[y' \succ_{i,x} y] := 1 -\boldsymbol{1}[y \succcurlyeq_{i,x} y'].
\end{align*}
Thus $\boldsymbol{1}[y \succcurlyeq_{i,x} y'] + \boldsymbol{1}[y' \succ_{i,x} y] = 1$ for all pairs of $(y,y')$. We extend the above function naturally to compare sets of responses $S, S' \subseteq \+Y$ for $\succ_{i,x}$ by comparing the most preferred response (the maximal element according to $\succ_{i,x}$) in each set. Formally, let $y_* \in S$ be the most preferred response in $S$, i.e., there is no other $y \in S$ such that $y \succ_{i,x } y_*$. Similarly, let $y_*' \in S'$ be the most preferred responses in $S'$. Then we define
\begin{align*}
    \boldsymbol{1}[S \succcurlyeq_{i,x} S'] = \boldsymbol{1}[y_* \succcurlyeq_{i,x} y_*'], \boldsymbol{1}[S \succ_{i,x} S'] = \boldsymbol{1}[y_* \succ_{i,x} y_*'].
\end{align*}

Given a population of users, a distribution $\+D$ over $\{\succ_i\}_{i \in [n]}$, the aggregated population preference $\-P_{\+D}$ is: for any $\pi, \pi'$ and $x$,
\begin{align*}
    \-P_{\+D} [ \pi \succcurlyeq \pi' \mid x] &:= \-E_{i \sim \+D}\-E_{S \sim \pi(\cdot \mid x), S' \sim \pi'(\cdot \mid x)} \InBrackets{ \boldsymbol{1}[S \succcurlyeq_{i,x} S']}, \\
    \-P_{\+D} [ \pi \succ \pi' \mid x] &:= \-E_{i \sim \+D}\-E_{S \sim \pi(\cdot \mid x), S' \sim \pi'(\cdot \mid x)} \InBrackets{ \boldsymbol{1}[S \succ_{i,x} S']}.
\end{align*}
We show in the following proposition that an aggregated preference $\-P_{\+D}$  from a population of rankings also has \Cref{dfn:general preferences}.
\begin{proposition}\label{prop:preference} $\-P_{\+D}$ satisfies \Cref{dfn:general preferences}.
\end{proposition}

\begin{remark}\label{remark:ranking}
Since \Cref{dfn:general preferences} applies to every prompt $x$, it also holds for the generalized preference model in which the population distribution $\+D$ may depend on $x$.\end{remark}

Throughout the paper, we assume a general population preference $\-P_D$ satisfying \Cref{dfn:general preferences}. 
This includes, in particular, mixtures of Plackett--Luce (PL) preferences, the standard social choice setting in which each user has a ranking over responses for each prompt, as well as generalizations in \Cref{remark:PL} and \Cref{remark:ranking}. Moreover, a more general class of \emph{random utility} preferences (\Cref{dfn:rum-preference}) also satisfies \Cref{dfn:general preferences}. We defer the discussion to \Cref{sec:multi-output-opponents} when we consider the extension of multi-output opponents.

\subsection{Alignment Methods}\label{sec:pre RLHF NLHF}
We briefly introduce two existing LLM alignment methods.
\paragraph{Reinforcement Learning From  Human Feedback (RLHF)} RLHF~\citep{christiano2017deep, bai2022training, ouyang2022training} first learns a reward function $r: \+X \times \+Y \rightarrow [0,1]$ according to the Bradley-Terry model from a preference dataset. Given a reference policy $\pi_{\mathrm{ref}}$, the RLHF policy is a solution that maximizes the reward subject to KL constraints: $\forall x \in \+X$, $\pi_{\mathrm{RLHF}}(\cdot \mid x)$ is
\begin{align*}
    \max_{\pi: \+X \rightarrow \Delta(\+Y)}  \-E_{y \sim \pi(\cdot\mid x)} \-E \InBrackets{r(x,y) - \eta \cdot \mathrm{KL}(\pi(\cdot \mid x), \pi_{\mathrm{ref}}(\cdot \mid x))}.
\end{align*}
A critical limitation of RLHF is that the BT model fails to capture diverse, possibly nontransitive, human preferences~\citep{munos2024nash}. 

\paragraph{Nash Learning From Human Feedback (NLHF)} Given a population preference $\-P_{\+D}$, NLHF~\citep{,munos2024nash,swamy2024minimaximalist} aims to find a Nash equilibrium policy $\pi_{\mathrm{NLHF}}: \+X \rightarrow \Delta(\+Y)$ of the following two-player \emph{symmetric} constant-sum game: $\forall x \in \+X$,
\begin{align}\label{eq:NLHF}
    \max_{\pi(\cdot\mid x) \in \Delta(\+Y)} \min_{\pi'(\cdot \mid x) \in \Delta(\+Y)} \-P_{\+D} \InBrackets{ \pi( \cdot \mid x) \succcurlyeq \pi'(\cdot \mid x)}
\end{align} 
The NLHF policy $\pi_{\mathrm{NLHF}}$ achieves a $\tfrac{1}{2}$ win rate against any other single-out model under $\+P_{\+D}$.

\section{U-Alignment via Multi-Output Policy: Possibility and Limitations}\label{sec:multi-output}
We note that by adapting existing results for the Condorcet winning set from the social choice theory literature, we can prove the possibility of \ualign with rate $\nicefrac{k}{k+1}$ using \emph{multi-output policies}. A $k$-output policy $\pi:\+X \rightarrow \Delta(\+Y^k)$ is randomized mapping that returns a $k$-set of responses for each prompt $x$. The following theorem is adapted from~\citep{cheng2020group, jiang2020approximately, charikar2025six}. Missing proofs in this section are in \Cref{app:multi-ouput}

\begin{theorem}[$(k,\nicefrac{k}{k+1})$-Robust Alignment Using a Multi-Output Policy]
\label{theorem:multioutput alignment}
For any $k\ge 1$ and any population preference $P_{\+D}$ satisfying \Cref{dfn:general preferences}, there exists a $k$-output policy $\pi:\mathcal{X}\to\Delta(\mathcal{Y}^k)$ such that for every prompt $x\in\mathcal{X}$ and every response $y\in\mathcal{Y}$,
$
P_{\+D}\left[\pi \succcurlyeq y \mid x\right]\ge \nicefrac{k}{k+1}. 
$
\end{theorem}
\begin{proof}
    Fix any $x \in \+X$. It suffices to prove that 
    \begin{align*}
        \max_{\pi(\cdot \mid x) \in \Delta(\+Y^k)}\min_{ y\in \+Y}  \-P_{\+D}[\pi \succcurlyeq y \mid x] = \max_{\pi(\cdot \mid x) \in \Delta(\+Y^k)}\min_{\pi'\in \Delta(\+Y)}  \-P_{\+D}[\pi \succcurlyeq \pi' \mid x] \ge  \frac{k}{k+1}
    \end{align*}
    By von Neumann's minimax theorem, we have
    \begin{align*}
        \max_{\pi(\cdot \mid x) \in \Delta(\+Y^k)}\min_{\pi'\in \Delta(\+Y)}  \-P_{\+D}[\pi \succcurlyeq \pi' \mid x] &=
        \min_{\pi'\in \Delta(\+Y)}\max_{\pi(\cdot \mid x) \in \Delta(\+Y^k)}  \-P_{\+D}[\pi \succcurlyeq \pi' \mid x] \\
        &\ge \min_{\pi'\in \Delta(\+Y)}\-P_{\+D}[(\pi')^{\otimes k} \succcurlyeq \pi' \mid x] \\
        &\ge \frac{k}{k+1},
    \end{align*}
    where the first inequality holds since the $k$-product distribution $(\pi')^{\otimes k}$ is a valid $k$-output policy, and the last inequality holds by \Cref{dfn:general preferences}.
\end{proof}

We remark that \Cref{theorem:multioutput alignment} is based on the following construction: any Nash equilibrium policy of the following two-player constant-sum game (a min-max optimization problem) gives a $(k, \nicefrac{k}{k+1})$-robust aligned multi-output policy: $\forall x \in \+X$,
\begin{align}\label{eq:multioutput game}
    \max_{\pi(\cdot \mid x) \in \Delta(\+Y^k)}\min_{\pi'\in \Delta(\+Y)}  \-P_{\+D}[\pi \succcurlyeq \pi' \mid x].
\end{align}
This two-player game objective~\eqref{eq:multioutput game} generalizes Nash Learning from Human Feedback (NLHF)~\citep{munos2024nash} in \eqref{eq:NLHF}, which is the special case of \eqref{eq:multioutput game} with $k =1$. 

\subsection{Limitations of Multi-Output Policies}\label{sec:limits of multi-output policy}
Although \Cref{theorem:multioutput alignment} guarantees the existence of a $(k, \nicefrac{k}{k+1})$-robust aligned multi-output policy, it does so by optimizing over general multi-output policies rather than standard single-output policies. This is unsatisfactory and impractical from both optimization and efficiency standpoints:
\begin{itemize}[leftmargin=*]
    \item 
    The space of $k$-output policies, i.e.,\ distributions in $\Delta(\mathcal{Y}^k)$, is vastly larger---and typically harder to optimize over---than the space of single-output policies $\Delta(\mathcal{Y})$. 
When $|\mathcal{Y}|=m$, a $k$-output policy is a distribution over $\mathcal{Y}^k$, whose support can be as large as $m^k$ (exponential in $k$), whereas a single-output policy is a distribution over $\mathcal{Y}$ with support size at most $m$.
    \item     Training a multi-output policy departs from standard practice, since mainstream LLM training pipelines are built for single-output policies. 
Although one could plausibly modify existing architectures to produce $k$ responses, such modifications would entail significant changes to the training pipeline, making them impractical in most settings.
    \item  The objective in \eqref{eq:multioutput game} is \emph{asymmetric}. As a result, self-play approaches to finding a Nash equilibrium must maintain and update two separate models, one for each player. This is less practical than maintaining and training a single model.
\end{itemize}
In light of the optimization and efficiency concerns for multi-output policies, together with the asymmetry of \eqref{eq:multioutput game}, existing results from social choice theory do not immediately carry over to the LLM alignment setting, where one typically trains a single-output policy.

\subsection{Tightness of the Rate $\nicefrac{k}{k+1}$}
We show that the rate $\nicefrac{k}{k+1}$ is tight with two lower bounds. The first lower bound assumes a Plackett-Luce (PL) preference model. We show that for any $k \ge 1$, there exists a PL preference such that it is impossible to achieve $(k, f(k))$-robust alignment with $f(k) > \nicefrac{k}{k+1}$. 
\begin{proposition}[Lower Bound for Plackett-Luce Model]\label{prop:lower bound PL}
    For any $k \ge 1$, there exists a Plackett-Luce preference $\-P$ such that for any policy $\pi: \+X \rightarrow \Delta(\+Y^k)$, there exists a prompt $x \in \+X$ and a response $y \in \+Y$ such that $\-P\InBrackets{\pi \succcurlyeq y | x} \le \nicefrac{k}{k+1}$.
\end{proposition}
The second lower bound is for a population of rankings. We show that for any $k \ge 1$, there exists a population preference over rankings such that the best-possible policy achieves $(k, f(k))$-robust alignment with $f(k) \le \nicefrac{k}{k+1}\cdot(1 + \nicefrac{1}{|\+Y|}) \rightarrow \nicefrac{k}{k+1}$ when the total number of responses $|\+Y|\rightarrow \infty$.

\begin{proposition}[Lower Bound for Rankings]\label{prop:lower bound rankings}
    For any $k \ge 1$, there exists a population preference over rankings $\-P_{\+D}$ such that for any policy $\pi: \+X \rightarrow \Delta(\+Y^k)$, there exists a prompt $x \in \+X$ and a response $y \in \+Y$ such that $\-P_{\+D}\InBrackets{\pi \succcurlyeq y \mid x} \le \nicefrac{k}{k+1}\cdot(1 + \nicefrac{1}{|\+Y|})$.
\end{proposition}

\section{U-Alignment via Single-Output Policy: Power of Test-Time Scaling}\label{sec:single-output}
In this section, we explore the power of standard single-output policies. In the context of alignment, the Nash learning from human feedback (NLHF) policy gives $\frac{1}{2}$ win rate against any other single-output policy, but the win rate can not be further improved. However, the barrier of $\frac{1}{2}$ for single-output policies is for sampling and evaluating \emph{one} response, but not for sampling $k$ responses in test time. A natural question is: 
\textit{can test-time scaling of a \textbf{single-output policy}} \textit{achieve \ualign with rate $\nicefrac{k}{k+1}$?} 

More formally, for any $k \ge 1$, does there exist a single-output policy $\pi$ such that the test-time scaled policy $\pi^{\otimes k}$ achieves $(k, \nicefrac{k}{k+1})$-robust alignment?
We remark that even proving the existence of such a policy is nontrivial. Although test-time scaling is clearly no worse than sampling a single response, it is unclear whether---and to what extent---a policy exhibits strong test-time performance. In particular, even if a model is optimal for $k=1$, test-time scaling of this model need not yield improved performance.

\subsection{Limits of Test-Time Scaling of NLHF}\label{sec:limits_of_nlhf_and_rlhf}

We illustrate this phenomenon by showing that test-time scaling the exact NLHF policy, even though it is optimal for $k=1$, need not improve alignment beyond $\tfrac{1}{2}$, even when the number of samples is very large. In particular, we prove a fundamental limitation for NLHF: for any $k>1$, the $k$-sample policy obtained by drawing $k$ i.i.d.\ responses from the NLHF policy cannot guarantee a win rate above $\tfrac{1}{2}$ except for an arbitrarily small slack. Our lower bound is established in a simple non-contextual setting, thereby strengthening the impossibility result.

\begin{proposition}[Limits of Test-Time Scaling of NLHF]\label{prop:limit of NLHF}
For any $\varepsilon \in (0,\tfrac{1}{2})$, there exist a response set $\+Y$, a reward function $r:\+Y\to\mathbb{R}$, and a population preference model $\-P_{\+D}$ such that the following statements hold the population preference $\-P_{\+D}$ admits a unique NLHF policy $\pi_{\mathrm{NLHF}}$. Moreover, for every $k\ge 1$, $
    \min_{y\in\+Y}\; \-P_{\+D}\big[\pi_{\mathrm{NLHF}}^{\otimes k}\succcurlyeq y\big]\ \le\ \tfrac{1}{2}+\varepsilon.
$
\end{proposition}

We note that the limitation of NLHF stems from its lack of diversity. In particular, when there exists a response $y$ that is preferred by a strict majority of the population (that is, $y$ is a Condorcet winner), the NLHF policy collapses to outputting $y$ with probability $1$. This is optimal when $k=1$, since it guarantees a win rate of at least $\tfrac{1}{2}$ against any fixed alternative. However, always outputting the majority-preferred response ignores minority's preferences and therefore provides no benefit from test-time scaling.

\subsection{Multi-Player Preference Optimization for U-Alignment}
Although NLHF does not exhibit the desired test-time scaling and, under single-sample evaluation, is limited by a $\tfrac{1}{2}$ barrier, there nonetheless exists a single-output policy $\pi$ whose test-time scaled policy $\pi^{\otimes k}$ achieves $(k,\nicefrac{k}{k+1})$-robust alignment. Moreover, we characterize such a policy as the symmetric Nash equilibrium of a simple $(k+1)$-player game.

We focus on the non-contextual setting; the contextual extension follows by applying the construction pointwise to each context. To this end, we introduce a \emph{multi-player alignment game}, which generalizes the two-player NLHF game.

\begin{definition}[Multi-Player Alignment Game and Nash equilibrium]\label{dfn:multi-player alignment game}
Fix $k\ge 1$ and a population preference model $\-P_{\+D}$ over $\+Y$.
The \emph{$(k+1)$-player alignment game} has player set $[k+1]$, and each player $j\in[k+1]$ chooses an action $\pi_j\in\Delta(\+Y)$.
Given a strategy profile $(\pi_j,\pi_{-j})$, where $\pi_{-j}=\{\pi_\ell\}_{\ell\in[k+1]\setminus\{j\}}$, the utility of player $j$ is
\begin{align}\label{eq:multi-game utility}
u_j(\pi_j,\pi_{-j})
:= \-P_{\+D}\InBrackets{ \pi_j \succ \bigotimes_{\ell\ne j}\pi_\ell },
\end{align}
which is the win rate of the single-output policy $\pi_j$ against the $k$-output product policy $\bigotimes_{\ell\ne j}\pi_\ell$. A \emph{Nash equilibrium} is a strategy profile $(\pi^*_1,\ldots,\pi^*_{k+1})$ such that no player can improve her utility by deviating unilaterally. That is, for every $j\in[k+1]$ and every $\pi\in\Delta(\+Y)$,
\[
u_j(\pi^*_j,\pi^*_{-j}) \ \ge\ u_j(\pi,\pi^*_{-j}).
\]
\end{definition}
We show that the symmetric Nash equilibrium of this game, which we call the \emph{Multi-Player Nash Equilibrium (MPNE)} policy, enjoys the desired test-time scaling property. 
Missing proofs in this section are deferred to \Cref{app:single-ouput}.
\begin{theorem}[Properties of MPNE policy of the Multi-Player Alignment Game]
\label{theorem:single-output good}
For any population preference $\-P_{\+D}$ satisfying \Cref{dfn:general preferences}, the $(k+1)$-player alignment game (\Cref{dfn:multi-player alignment game}) admits a symmetric Nash equilibrium in which every player uses the same policy $\pi^*$. Moreover, the test-time scaled policy $(\pi^*)^{\otimes k}$ achieves $(k,\nicefrac{k}{k+1})$-robust alignment:
$
\min_{\pi}\; \-P_{\+D}\InBrackets{(\pi^*)^{\otimes k} \succcurlyeq \pi}
\ \ge\ \nicefrac{k}{k+1}.
$
\end{theorem}

\begin{remark}
    \Cref{theorem:single-output good} holds for all preferences satisfying \Cref{dfn:general preferences} (item 1 to 3). An independent work by~\citet{charikar2025metric} has a similar result for the ranking setting. We were aware of that work only after the conference version of the paper. Interestingly, their proof uses a non-symmetric two-player game and the minimax theorem, while our proof is through the multiplayer alignment game without the use of the minimax theorem.
\end{remark}
\begin{proof}
The $(k+1)$-player alignment game is symmetric, and each player's strategy space $\Delta(\+Y)$ is compact and convex, with utilities continuous and linear in each argument. Hence, by standard existence results for symmetric games, the game admits a symmetric Nash equilibrium in which every player uses the same policy $\pi^*$.

By the definition of Nash equilibrium, no player can improve her utility by deviating unilaterally. In particular, for any player $j$,
\[
\max_{\pi} u_j(\pi,(\pi^*)^{\otimes k})
\ \le\
u_j(\pi^*,(\pi^*)^{\otimes k}).
\]
By \Cref{dfn:general preferences}, the win rate of $\pi^*$ against $k$ independent copies of itself is at most $\frac{1}{k+1}$, so
\[
u_j(\pi^*,(\pi^*)^{\otimes k}) \le \frac{1}{k+1}.
\]
Therefore,
\begin{align*}
\min_{\pi}\; \-P_{\+D}\InBrackets{(\pi^*)^{\otimes k} \succcurlyeq \pi}
&= 1 - \max_{\pi}\; \-P_{\+D}\InBrackets{\pi \succ (\pi^*)^{\otimes k}} \\
&\ge 1 - u_j(\pi^*,(\pi^*)^{\otimes k}) \\
&\ge 1 - \frac{1}{k+1}.
\end{align*}
This proves that $(\pi^*)^{\otimes k}$ achieves the claimed robust alignment guarantee.
\end{proof}

We are now ready to state the paper's main theorem.

\begin{restatable}{theorem}{MainTheorem}\label{theorem:main}
For any population preference $\-P_{\+D}$ satisfying \Cref{dfn:general preferences}, there exists a family of single-output policies $\{\pi_k\}_{k\in\mathbb{N}_{>0}}$ such that their test-time scaled policies $\{\pi_k^{\otimes k}\}_{k\in\mathbb{N}_{>0}}$ achieve \ualign at the optimal rate $\nicefrac{k}{k+1}$. In particular, for every $k\ge 1$, there exists a single-output policy $\pi_k:\+X\to\Delta(\+Y)$ satisfying
\[
\min_{\pi'}\; \-P_{\+D}\InBrackets{\pi_k^{\otimes k} \succcurlyeq \pi'} \ \ge\ \nicefrac{k}{k+1}.
\]
\end{restatable}

\begin{proof}
Optimality follows from \Cref{prop:lower bound PL,prop:lower bound rankings}, which show that the rate $\nicefrac{k}{k+1}$ cannot be improved, even when allowing multi-output policies. The existence of the stated family of single-output policies $\{\pi_k\}_{k\in\mathbb{N}_{>0}}$ follows directly from \Cref{theorem:single-output good}.
\end{proof}

A few remarks are in order. Compared with existing alignment frameworks such as RLHF and NLHF, the new alignment framework formulates the LLM alignment problem as a multi-player alignment game (\Cref{dfn:multi-player alignment game}) and aims to find the multi-player Nash equilibrium (MPNE) policy. We briefly discuss the advantages of MPNE over RLHF and NLHF.

\begin{itemize}[leftmargin=*]
    \item[1.] A simple test-time sampling MPNE policy achieves an optimal $(k, \nicefrac{k}{k+1})$-robust alignment guarantee. In contrast, test-time scaling of NLHF only yields a $(k,\nicefrac{1}{2}+\varepsilon)$-robust alignment guarantee. 
    \item[2.] As a consequence of its strong test-time scaling property, the MPNE policy can generate diverse responses that accommodate both the preferences of the majority and the minority. In contrast, the NLHF/RLHF policy may collapse into a near-deterministic policy that favors the majority.
\end{itemize}

In \Cref{sec:uniqueness}, we show that MPNE are not necessary for optimal \ualign, in the sense that there may exist optimal U-aligned policies that are not equilibria of the alignment game. Moreover, the game may admit asymmetric Nash equilibria that differ from the symmetric MPNE.

\subsection{Test-Time Performance of a Fixed Robust Policy}

Suppose $\pi_k$ achieves $(k,\nicefrac{k}{k+1})$-robust alignment.
What is its performance when, at test time, we evaluate using $\ell \neq k$ samples?
We show that the test-time performance of the product policy $\pi_k^{\otimes \ell}$ degrades gracefully as $\ell$ decreases, and maintains the win-rate of $\nicefrac{k}{k+1}$ once $\ell \ge k$.

\begin{theorem}\label{thm:fixed scaling}
For any population preference $\-P_{\+D}$ satisfying \Cref{dfn:general preferences} and any policy $\pi_k$ achieving $(k,\nicefrac{k}{k+1})$-robust alignment, we have for all $\ell \ge 1$,
\[
\min_{\pi:\+X \rightarrow \Delta(\+Y)} \-P_{\+D}[\pi_k^{\otimes \ell} \succcurlyeq \pi]
\;\ge\;
\frac{\min\{k,\ell\}}{k+1}.
\]
\end{theorem}

This result shows that robustness confers additional test-time stability:
a policy optimized for a sample size $k$ continues to guarantee meaningful performance even when evaluated with fewer samples $\ell < k$, and achieves its full robust-alignment guarantee whenever $\ell \ge k$.

\subsection{Theoretical Guarantees on Convergence of Self-Play Learning Dynamics}

We study the convergence properties of self-play no-regret learning dynamics in multi-player alignment games (\Cref{dfn:multi-player alignment game}) to the desirable multi-player Nash equilibrium (MPNE) policy. When $k=1$, the resulting two-player alignment game is a constant-sum game and coincides with Nash learning from human feedback (NLHF). In this case, a Nash equilibrium policy can be learned efficiently via self-play using no-regret algorithms~\citep{munos2024nash, swamy2024minimaximalist}, and algorithms with last-iterate convergence guarantees are known~\citep{liu2024comal, wang2025magnetic}. Moreover, these no-regret methods are practical to implement and have been successfully applied to large-scale LLM alignment with strong empirical performance~\citep{wu2025selfplay, zhang2025iterative, liu2024comal}.

When $k>1$, the $(k+1)$-player alignment game becomes substantially more challenging to solve. In this regime, there are currently no general theoretical guarantees for the convergence of no-regret self-play dynamics to Nash equilibria. In the remainder of this section, we establish two theoretical guarantees for self-play learning dynamics in multi-player alignment games.

Our first guarantee shows that any symmetric Nash equilibrium policy is a fixed point of gradient-based learning dynamics. Since the game is symmetric, all players apply the same update rule. Accordingly, we write $\pi^t$ for the common policy at iteration $t\ge 1$. This result implies that if the gradient-ascent dynamic converges, then its limit point must be a symmetric Nash equilibrium of the alignment game.

\begin{proposition}[Fixed Point of Projected Gradient Ascent]\label{proposition:fixed point}
Let $u_1(\pi_1,\pi_{-1})$ be the utility of player $1$ in the $(k+1)$-player alignment game.
For any step size $\eta > 0$, define the projected gradient-ascent operator $F_{\eta}:\Delta(\+Y)\to\Delta(\+Y)$ by
$
F(\pi):=\Pi_{\Delta(\+Y)}\Big(\pi+\eta\cdot\nabla_{\pi_1} u_1(\pi,\pi^{\otimes k})\Big),
$
where $\Pi_{\Delta(\+Y)}$ denotes Euclidean projection onto the simplex.
Then $F_{\eta}$ has at least one fixed point. Moreover, any fixed point $\pi$ of $F_{\eta}$ induces a symmetric Nash equilibrium $(\pi,\ldots,\pi)$ of the multi-player alignment game.
\end{proposition}

Our second result provides finite-time guarantees for no-regret self-play learning dynamics~\citep{cesa2006prediction}, at the cost of storing a sequence of past policies. Let each player employ the same online learning algorithm. Since the alignment game is symmetric, each player produces the same sequence of strategies $\{\pi^t\}_{t \in [T]}$. Recall that each player's utility function is $u(\cdot, (\pi^t)^{\otimes k}) = \-P_{\+D}(\cdot \succ \sigma)$, which is linear. The regret of each player is 
\begin{align*}
    \reg^T := \max_{\pi \in \Delta(\+Y)} \sum_{t=1}^T u(\pi, (\pi^t)^{\otimes k}) - \sum_{t=1}^T u(\pi^t, (\pi^t)^{\otimes k}).
\end{align*}

Consider any no-regret learning algorithm with regret $\reg^T$. We claim that the policy which samples $t\sim \Unif([T])$ and then executes $(\pi^t)^{\otimes k}$ achieves $(k,\nicefrac{k}{k+1}-\nicefrac{\reg^T}{T})$-robust alignment.

\begin{proposition}[No-Regret Self-Play Dynamics]\label{prop:self-play}
Suppose the players self-play using a no-regret learning algorithm and generate iterates
$\{\pi^t\}_{t\in[T]}$ with regret $\reg^T$. Let $\sigma^T$ denote the policy that samples
$t\sim\Unif([T])$ and then executes $(\pi^t)^{\otimes k}$. Then $\sigma^T$ achieves
$(k,\nicefrac{k}{k+1}-\nicefrac{\reg^T}{T})$-robust alignment.
\end{proposition}

We note that many no-regret algorithms, such as online gradient ascent and multiplicative weights update, achieve $\reg^T = O(\sqrt{T})$~\citep{hazan2016introduction,orabona2019modern}. Consequently, taking $T = O(\nicefrac{1}{\varepsilon^2})$ iterations suffices to obtain a policy that achieves $(k,f(k))$-robust alignment with $f(k)\ \ge\ \nicefrac{k}{k+1}-\varepsilon$. In \Cref{sec:last-iterate}, we show that some no-regret learning algorithm enjoys the stronger notion of \emph{last-iterate} convergence, the generated policy \emph{i.e.,} $\lim_{t \rightarrow \infty} \pi^t = \pi^*$ converges to a Nash policy without the time-average mixing.

\section{Extensions to Multi-Output Opponents}
\label{sec:multi-output-opponents}

We extend the definitions of robust alignment and \ualign to allow the opponent
policy to generate $\ell\ge 1$ responses. This strictly generalizes the 
case $\ell=1$.

\subsection{Arbitrary Multi-Output Opponents}
We first consider the case where the opponent can be abitrary multi-output policy.
\begin{definition}[$(k,\ell,f_\ell(k))$-Robust Alignment]
Let $k,\ell \in \mathbb{N}_{>0}$ denote the numbers of responses generated at
test time by the model and the opponent, respectively. Let
$f_\ell:\mathbb{N}_{>0}\to[0,1]$ be a function. A policy
$\pi:\+X\to\Delta(\+Y^k)$ is said to achieve
$(k,\ell,f_\ell(k))$-robust alignment if, for every prompt $x\in\+X$, its win
rate against any opponent policy $\pi':\+X\to\Delta(\+Y^\ell)$ satisfies
\[
    \min_{\pi':\+X\to\Delta(\+Y^\ell)}
    \-P_{\+D}\InBrackets{\pi \succcurlyeq \pi' \mid x}
    \ge f_\ell(k).
\]
\end{definition}

We note that $(k,f(k))$-robust alignment is the special case $\ell=1$ of the
above definition. Equipped with this extension, we introduce the corresponding
notion of $\ell$-\ualign.

\begin{definition}[$\ell$-Asymptotic \ualign]
Fix $\ell \in \mathbb{N}_{>0}$ as the number of responses generated at test
time by the opponent model. Let $f_\ell:\mathbb{N}_{>0}\to[0,1]$ be a rate
function satisfying $\lim_{k\to\infty} f_\ell(k)=1$. We say that a family of
policies $\{\pi_k\}_{k\in\mathbb{N}_{>0}}$ achieves
\emph{$\ell$-\ualign with rate $f_\ell$} if, for every
$k\in\mathbb{N}_{>0}$, the policy $\pi_k$ achieves
$(k,\ell,f_\ell(k))$-robust alignment.
\end{definition}

In \Cref{theorem:main}, we show that when $\ell=1$ we can achieve $1$-\ualign with rate $f_1(k)=\nicefrac{k}{k+1}$. A natural question is: what is the optimal rate for $\ell$-\ualign? On the lower bound side, we can adapt the proof of \Cref{prop:lower bound PL} to obtain a lower bound of $\nicefrac{k}{k+\ell}$ (for all $k\ge 1$) under the Plackett--Luce model. On the upper bound side, combining \Cref{theorem:main} with a union bound shows that the Nash equilibrium policy of the multi-player alignment game (\Cref{dfn:multi-player alignment game}) already achieves $\ell$-\ualign with rate $\nicefrac{k+1-\ell}{k+1}$. In particular, if the opponent outputs $\ell$ responses and we seek a $\tfrac{1}{2}$ win rate guarantee, then it suffices for the model to output $2\ell$ responses.  The upper and lower bounds coincide for $\ell=1$, but a gap remains for $\ell>1$. Closing this gap is an interesting direction for future work.

\begin{theorem}[Arbitrary Multi-Output Opponents]
\label{thm:arbitrary-multi-output-opponent}
For any population preference $\-P_{\+D}$ satisfying
\Cref{dfn:general preferences}, there exists a family of single-output policies
$\{\pi_k\}_{k\in\mathbb{N}_{>0}}$ such that their test-time scaled policies
$\{\pi_k^{\otimes k}\}_{k\in\mathbb{N}_{>0}}$ achieve $\ell$-\ualign with rate
\[
    f^{\mathrm{arb}}_\ell(k)
    :=
    \max\left\{0,\frac{k+1-\ell}{k+1}\right\}.
\]
In particular, $\pi_k$ can be chosen as any symmetric Nash equilibrium policy of
the $(k+1)$-player alignment game in \Cref{dfn:multi-player alignment game}.
\end{theorem}

\subsection{Product Opponents} In the above, we allow the opponent to use an arbitrary distribution in
$\Delta(\+Y^\ell)$, including correlated or deterministic multisets. In many
test-time scaling settings, however, a more natural comparison is against an
opponent that also uses a single-output policy and generates $\ell$ independent
samples. We next study this \emph{product-opponent} benchmark. It is weaker than
the arbitrary multi-output opponent benchmark, but it is the direct analogue of
test-time scaling for both sides. Moreover, when $\ell = 1$, it also recovers the special case of $(k, f(k))$-robust alignment.

\begin{definition}[Product-Opponent Robust Alignment and U-Alignment]
\label{dfn:product-opponent-robust}
Fix $k,\ell\ge 1$. A single-output policy
$\pi:\+X\to\Delta(\+Y)$ achieves \emph{product-opponent $(k,\ell,f_\ell(k))$-robust alignment}
if, for every prompt $x\in\+X$,
\[
    \min_{\pi':\+X\to\Delta(\+Y)}
    \-P_{\+D}\InBrackets{\pi^{\otimes k}
    \succcurlyeq (\pi')^{\otimes \ell} \mid x}
    \ge f_\ell(k) .
\]
Let $f_\ell:\mathbb{N}_{>0}\to[0,1]$ be a rate
function satisfying $\lim_{k\to\infty} f_\ell(k)=1$. We say that a family of
policies $\{\pi_k\}_{k\in\mathbb{N}_{>0}}$ achieves
\emph{product-opponent $\ell$-\ualign with rate $f_\ell$} if, for every
$k\in\mathbb{N}_{>0}$, the policy $\pi_k$ achieves
product-opponent $(k,\ell,f_\ell(k))$-robust alignment.
\end{definition}

\subsubsection{Characterization via Convex-Concave Min-Max Optimization} 
In this section, we address the optimal rate for product-opponent $\ell$-\ualign. Somewhat surprisingly, we show that 
the optimal policy can be obtained by solving a two-player zero-sum game with a convex-concave objective. Such a characterization gives a tight rate of $\frac{k}{k+\ell}$ for product-opponent $\ell$-\ualign and enables efficient learning dynamics with last-iterate convergence rate. 


Specifically, fix a prompt $x\in\+X$, for $p,q\in\Delta(\+Y)$, define the product-opponent payoff
\begin{align}
    \Phi_{x,k,\ell}(p,q)
    :=
    \-P_{\+D}\InBrackets{p^{\otimes k}
    \succcurlyeq q^{\otimes \ell}\mid x}.
    \label{eq:product-opponent-payoff}
\end{align}
Equivalently, if $S\sim p^{\otimes k}$ and
$T\sim q^{\otimes \ell}$ are independent multisets, then
\[
    \Phi_{x,k,\ell}(p,q)
    =
    \-E\InBrackets{K_x(S,T)},
    \text{where }
    K_x(S,T):=\-P_{\+D}\InBrackets{S\succcurlyeq T\mid x}.
\]

For generalization, we consider here a broader class of \emph{random utility} preferences, which includes both the mixtures of PL preferences and the populations of rankings in \Cref{sec:prelim} as special cases. Since the prompt $x$ is fixed, we suppress it from the notation; in particular, the distributions below may depend on $x$.

\begin{definition}[Random Utility Preferences]
\label{dfn:rum-preference}
A population preference $\-P_{\+D}$ is a \emph{random utility preference} if it is specified by a distribution $\+D$ over user types $\omega$ and, for every type $\omega$ and response $y\in\+Y$, a probability distribution $Q_{\omega,y}$ over $\R$, as follows. Given two multisets of responses $S$ and $T$, a type $\omega\sim\+D$ is drawn independently of $S$ and $T$, and, conditional on $\omega$, every copy $j$ of a response $y_j$ in $S+T$ independently receives a utility $V_j\sim Q_{\omega,y_j}$. Then
\[
    \-P_{\+D}\InBrackets{S\succcurlyeq T\mid x}
    :=
    \Pr\InBrackets{\max_{j\in S}V_j\ge\max_{j\in T}V_j},
    \qquad
    \-P_{\+D}\InBrackets{S\succ T\mid x}
    :=
    \Pr\InBrackets{\max_{j\in S}V_j>\max_{j\in T}V_j},
\]
where $j\in S$ ranges over the copies in $S$, counted with multiplicity. The preference between policies is defined by averaging over their sampled multisets as in \Cref{sec:prelim}. Under \emph{randomized tie-breaking}, we instead select a copy uniformly at random among the copies in $S+T$ with the maximal utility, and $S$ wins if the selected copy belongs to $S$.
\end{definition}
The distribution $Q_{\omega,y}$ depends only on the type $\omega$ and the response $y$: it is the same for every copy of $y$, and it does not depend on the other responses being compared or on whether the copy belongs to $S$ or $T$. The utilities of different copies can still be dependent through the common type $\omega$.

\begin{remark}[Examples of Random Utility Preferences]
\label{remark:rum-examples}
\Cref{dfn:rum-preference} includes the following preference models; see \Cref{app:rum-examples} for details.
\begin{itemize}[leftmargin=*]
    \item \textnormal{Mixtures of PL preferences:} for a PL type $i$ with reward function $r_i$, let $Q_{i,y}$ be the distribution of $r_i(x,y)+\xi$, where $\xi$ follows the standard Gumbel distribution. This is the standard random utility construction of Luce's choice probabilities~\citep{train2009discrete}, and it recovers exactly the mixtures of PL preferences in \Cref{sec:prelim}.
    \item \textnormal{Thurstone--Mosteller preferences:} $Q_{\omega,y}=\N(\mu_\omega(y),\sigma_\omega(y)^2)$ with $\sigma_\omega(y)>0$, i.e., every copy of $y$ receives the mean utility $\mu_\omega(y)$ plus fresh Gaussian noise~\citep{train2009discrete,vojnovic2017parameter}. This model allows arbitrary mixtures over types and hence over means and variances.
    \item \textnormal{Populations of rankings and correlated utilities:} let the type $\omega=(\omega_y)_{y\in\+Y}\in\R^{\+Y}$ be a random utility vector with an arbitrary joint distribution (e.g., a Gaussian vector with an arbitrary covariance matrix), and let $Q_{\omega,y}$ be the point mass at $\omega_y$, so that all copies of $y$ share the same utility $\omega_y$. Taking $\omega_y$ to be minus the position of $y$ in a ranking recovers exactly the populations of rankings in \Cref{sec:prelim}, including the weak-tie convention.
\end{itemize}
\end{remark}

\Cref{dfn:rum-preference} also specifies how repeated copies of a response are evaluated. In the PL and Thurstone--Mosteller models, every copy receives fresh noise, whereas in the ranking model, all copies of a response share the same utility. These are different extensions of the same pairwise preferences to multisets, and our results cover both. For example, if all utilities are i.i.d.\ and continuous, then $\{A,A\}$ beats $\{B\}$ with probability $\nicefrac{2}{3}$ when each copy receives a fresh utility, but with probability $\nicefrac{1}{2}$ when the two copies of $A$ share the same utility. Our analysis requires the utilities of different copies to be conditionally independent given the type $\omega$ and the identities of the responses. This is not a restriction for response-level utilities, since the whole utility vector can be included in the type as in the ranking example above, but it excludes arbitrary correlations across copies and utilities that depend on the other responses being compared. Finally, every random utility preference also satisfies \Cref{dfn:general preferences} (see \Cref{app:rum-property}), so our results in \Cref{sec:multi-output,sec:single-output} and \Cref{thm:arbitrary-multi-output-opponent} also apply to this broader class.

We first prove that the product-opponent payoff function is convex-concave for every random utility preference, including both the PL preference model and the ranking model.
\begin{proposition}[Concave--Convex Product-Opponent Game]
\label{prop:product-game-concave-convex}
For every $k,\ell\ge 1$ and every prompt $x\in\+X$, the map $
    (p,q)\mapsto \Phi_{x,k,\ell}(p,q)$
is concave in $p$ and convex in $q$ whenever $\-P_{\+D}$ is a random utility preference (\Cref{dfn:rum-preference}), e.g., a mixture of
Plackett--Luce preferences or a population of rankings as in \Cref{sec:prelim}. The same holds under randomized tie-breaking.
\end{proposition}

Based on the above result, we next show that the minimax solutions to the product-opponent game give a rate of $\nicefrac{k}{k+\ell}$ for product-opponent $\ell$-U-alignment, improving the $\nicefrac{k+1-\ell}{k+1}$ guarantee against arbitrary opponents. Moreover, the rate of $\nicefrac{k}{k+\ell}$ is tight.

\begin{theorem}[Tight Rate for Product Opponents]
\label{thm:sharp-product-opponent}
Suppose that $\-P_{\+D}$ is a random utility preference (\Cref{dfn:rum-preference}), e.g., a mixture of Plackett--Luce preferences or a
population of rankings. Then for every $k,\ell\ge 1$ and every prompt
$x\in\+X$,
\[
    \max_{p\in\Delta(\+Y)}
    \min_{q\in\Delta(\+Y)}
    \-P_{\+D}\InBrackets{p^{\otimes k}
    \succcurlyeq q^{\otimes \ell}\mid x}
    \ge
    \frac{k}{k+\ell}.
\]
Consequently, for every $k,\ell\ge 1$, there exists a single-output policy
$\pi_{k,\ell}:\+X\to\Delta(\+Y)$ that achieves product-opponent
$(k,\ell,\nicefrac{k}{k+\ell})$-robust alignment. Moreover, the constant
$\nicefrac{k}{k+\ell}$ is tight in the worst case over Plackett--Luce
preferences.
\end{theorem}

In fact, when ties occur with probability zero, e.g., for mixtures of PL or Thurstone--Mosteller preferences, the bound in \Cref{thm:sharp-product-opponent} holds with equality for every instance: the value of the product-opponent game is exactly $\nicefrac{k}{k+\ell}$, and the model and the opponent share a common optimal policy (\Cref{lemma:symmetric-saddle} in \Cref{sec:last-iterate}). When ties can occur, as in the ranking model, the value can be larger than $\nicefrac{k}{k+\ell}$.


\subsubsection{Learning Dynamics with Last-Iterate Convergence}
\label{sec:last-iterate}
In this subsection, we study learning dynamics for computing the policy with the optimal product-opponent $\ell$-U-alignment. We note that the convex-concave formulation above already allows one to run self-play learning algorithms and achieve \emph{last-iterate convergence} to Nash equilibrium. However, the naive implementation requires one maintain two sequences of policies $\{p^t, q^t\}$, which is inefficient in practice. 
Can we approach optimal product-opponent $\ell$-U-alignment maintaining only one sequence of policies? It seems impossible since when $k \ne \ell$, the game is not symmetric. However, we give an affirmative answer by showing a generalized symmetry in the convex-concave formulation along the diagonal: when $p^t = q^t$, their gradients are symmetric (\Cref{lemma:diagonal-gradient-identity}) so that any gradient-based learning algorithms produces the same $p^{t+1} = q^{t+1}$. By induction, we could just maintain one sequence of policies $\{p^t\}$.

We first introduce the product-diagonal condition, which is satisfied by every random utility preference (\Cref{dfn:rum-preference}) as explained below.

\begin{definition}[Product-Diagonal Condition]
\label{dfn:product-diagonal-condition}
Fix $k,\ell\ge 1$ and a prompt $x\in\+X$. We say that
$\Phi_{x,k,\ell}$ satisfies the \emph{product-diagonal condition} with value
$c_{k,\ell}$ if: (1) for every $q\in\Delta(\+Y)$, the map
    $p\mapsto \Phi_{x,k,\ell}(p,q)$ is concave; (2) for every $p\in\Delta(\+Y)$, the map
    $q\mapsto \Phi_{x,k,\ell}(p,q)$ is convex; (3)  for every $r\in\Delta(\+Y)$,
    \[
        \Phi_{x,k,\ell}(r,r)=c_{k,\ell}.
    \]
\end{definition}

We remark that every random utility preference satisfies \Cref{dfn:product-diagonal-condition} with $c_{k,\ell} = \nicefrac{k}{k+\ell}$.
\begin{itemize}
    \item If ties among the sampled copies occur with probability zero, e.g., for mixtures of Plackett--Luce preferences and Thurstone--Mosteller preferences, \Cref{dfn:product-diagonal-condition} holds with $c_{k,\ell}=\nicefrac{k}{k+\ell}$: the concave--convex property follows from \Cref{prop:product-game-concave-convex}, and the diagonal identity follows by exchangeability among the $k+\ell$ i.i.d. samples (see the proof of \Cref{thm:sharp-product-opponent}).
    \item If ties can occur, e.g., for population rankings with the weak-tie convention from \Cref{sec:prelim}, where all copies of a response share the same utility, the diagonal value may be larger than $\nicefrac{k}{k+\ell}$ because ties help the model. The exact diagonal identity holds if ties among the best sampled copies are broken uniformly at random (\Cref{dfn:rum-preference}), and \Cref{prop:product-game-concave-convex} also applies under this randomized tie-breaking. In the following, we will work with this randomized tie-breaking convention. We emphasize that guarantees for this randomized tie-broken comparison transfer to the weak-tie evaluation because weak ties only increase the model's win rate.
\end{itemize} 

Throughout the rest of this subsection, fix $x,k,\ell$ and assume the product-diagonal condition (\Cref{dfn:product-diagonal-condition}). Let $
    g_{x,k,\ell}(r)
    :=
    \nabla_1 \Phi_{x,k,\ell}(r,r)
$
denote the model-side diagonal gradient, where gradients are interpreted on the
tangent space of the simplex. Equivalently, for each $a\in\+Y$,
\[
    g_{x,k,\ell}(r)(a)
    =
    k\,
    \-E_{\substack{
        M\sim r^{\otimes(k-1)}\\
        T\sim r^{\otimes \ell}
    }}
    \InBrackets{K_x(M+\{a\},T)}.
\]
Thus $g_{x,k,\ell}(r)$ can be estimated on-policy: we can draw i.i.d. samples from policy $r$ and compute the empirically estimated gradient

Our main observation is that due to \Cref{dfn:product-diagonal-condition}, the two players' gradients along any diagonal strategies $(r,r)$ are \emph{symmetric}, \emph{i.e.,} they are inverse to each other. We note that this generalizes the symmetry observed in the special case of Nash learning from human feedback where both players output one response, to the general case where both players could output multiple responses.
\begin{lemma}[Diagonal Gradient Identity]
\label{lemma:diagonal-gradient-identity}
Assume \Cref{dfn:product-diagonal-condition}. Then for every
$r\in\Delta(\+Y)$,
\[
    \nabla_1\Phi_{x,k,\ell}(r,r)
    +
    \nabla_2\Phi_{x,k,\ell}(r,r)
    =
    0
\]
on the tangent space of the simplex. Equivalently, $
    \nabla_2\Phi_{x,k,\ell}(r,r)
    =
    -g_{x,k,\ell}(r)$.
\end{lemma}
\begin{proof}[Proof of \Cref{lemma:diagonal-gradient-identity}]
Let $h$ be any tangent direction of the simplex. For every sufficiently small
$\varepsilon$, $r+\varepsilon h\in\Delta(\+Y)$ and, by
\Cref{dfn:product-diagonal-condition},
\[
    \Phi_{x,k,\ell}(r+\varepsilon h,r+\varepsilon h)=c_{k,\ell}.
\]
Differentiating at $\varepsilon=0$ gives
\[
    \left\langle \nabla_1\Phi_{x,k,\ell}(r,r),h\right\rangle
    +
    \left\langle \nabla_2\Phi_{x,k,\ell}(r,r),h\right\rangle
    =
    0.
\]
Since this holds for every tangent direction $h$, the two gradients are
opposites on the tangent space. Additive constant vectors are irrelevant on the
simplex because they have zero inner product with every tangent direction.
\end{proof}

The product-diagonal condition also implies that the min-max problem admits a symmetric solution, i.e., the model and the opponent can use the same optimal policy. The proof is deferred to \Cref{app:multi-output-opponents}.
\begin{lemma}[Symmetric Saddle Point]
\label{lemma:symmetric-saddle}
Assume \Cref{dfn:product-diagonal-condition}. Then
\[
    \max_{p\in\Delta(\+Y)}\min_{q\in\Delta(\+Y)}\Phi_{x,k,\ell}(p,q)
    =
    \min_{q\in\Delta(\+Y)}\max_{p\in\Delta(\+Y)}\Phi_{x,k,\ell}(p,q)
    =
    c_{k,\ell}.
\]
Moreover, every maximin-optimal policy $r\in\argmax_{p}\min_{q}\Phi_{x,k,\ell}(p,q)$ is also minimax-optimal, and
\[
    \Phi_{x,k,\ell}(p,r)\ \le\ c_{k,\ell}\ \le\ \Phi_{x,k,\ell}(r,q)
    \qquad\text{for all } p,q\in\Delta(\+Y),
\]
that is, $(r,r)$ is a saddle point of $\Phi_{x,k,\ell}$.
\end{lemma}

Due to the generalized symmetry, if both player starts with the same policy $p^0 = q^0$ and use the same gradient-based learning algorithm, then their subsequent strategies would be exactly the same $p^t =q^t$. Therefore, we could maintain just one sequence of strategies $\{p^t\}$ during self-play training. More formally, we could use the following self-play training pipeline to compute single-output policies with product-opponent $\ell$-U-alignment with the optimal rate $\nicefrac{k}{k+\ell}$: start with the initial policy $\pi^1 = \pi_{\mathrm{init}}$, at every iteration $t \ge 1$
\begin{itemize}
    \item[1.] Sample responses from $\pi^t$ and estimate the gradient $g_{x,k,\ell}(\pi^t)$ for each $x$ in the training set.
    \item[2.] Update $\pi^{t+1}$ according to the preference optimization algorithm COMAL~\citep{liu2026comal}.\footnote{One can also use other preference optimization algorithms.}
\end{itemize}
We present a practical implementation of COMAL with the REBEL loss~\citep{gao2024rebel} in \Cref{alg:comal-rebel} in \Cref{app:algorithm}.

We remark COMAL~\citep{liu2026comal} is a general framework for min-max LLM optimization that itself allows implementation with many existing preference optimization pipelines such as PPO, GRPO, INPO, SPPO. Moreover, COMAL guarantees converges to the set of Nash equilibrium in the last-iterate. Formally, we have the following guarantees.
\begin{corollary}[Adapted from Theorem 1 in \citep{liu2026comal}]
\label{corollary:last-iterate}
    Fix $x, k, \ell$ and assume \Cref{dfn:product-diagonal-condition}. Let $p^*$ be any solution to the min-max optimization problem $\max_{p}\min_{q} \Phi_{x, k, \ell}(p,q)$ (by \Cref{lemma:symmetric-saddle}, $(p^*,p^*)$ is a saddle point) and $\pi_{\mathrm{init}}$ be any initial policy with the same support as $p^*$. Let $\{\pi^t\}$ be the policies produced by the COMAL algorithm. Then 
    \begin{itemize}
        \item[1.] $\mathrm{KL}(p^\star\,\|\,\pi^{t+1})
        \le
        \mathrm{KL}(p^\star\,\|\,\pi^{t})
        -
        \mathrm{KL}(\pi^{t+1}\,\|\,\pi^{t}), \forall t\ge 1;$
        \item[2.] The sequence $\{\pi^t\}$ converges to a solution $\pi^*$ of $\max_{p}\min_{q} \Phi_{x, k, \ell}(p,q)$ with the guarantee that
        \[
        \min_{q \in \Delta(\+Y)} \-P_{\+D} \InBrackets{(\pi^*)^{\otimes k}
    \succcurlyeq q^{\otimes \ell}\mid x}\ge c_{k,\ell}
        \]
    \end{itemize}
\end{corollary}
We leave the evaluation of this algorithm's performance in large-scale LLM-based experiments to future work. 

\section{Conclusion and Discussion}

In this work, we propose a new framework for tractable, pluralistic alignment via test-time scaling. 
We introduce a new objective, U-alignment, characterize its fundamental limits, identify limitations of existing approaches, and establish the existence of optimal U-aligned policies through a multi-player generalization of Nash learning from human feedback (NLHF). 
Our main result shows that with \(k\) test-time samples, a single-output policy can achieve a \(\nicefrac{k}{k+1}\) win rate against any other single-output policy.

\paragraph{Limitations and discussions.}
Our theoretical results on self-play learning dynamics have several limitations. 
First, \Cref{proposition:fixed point} and \Cref{corollary:last-iterate} only establish convergence in the limit. 
Second, \Cref{prop:self-play} requires maintaining multiple models simultaneously. 
Relatedly, our work does not provide experimental results on LLM alignment. However, we believe that no-regret learning dynamics for the multi-player alignment game can be implemented practically. 
First, since the game is symmetric, one only needs to maintain a single policy, since the update rule is the same for all players. 
Gradient estimation requires sampling additional responses from the current policy, but this overhead is manageable for moderate values of \(k\), which already yield non-trivial theoretical improvements. Second, existing practical implementations of no-regret-style algorithms in the LLM setting can be adapted to our multi-player alignment game by changing the utility function, or equivalently the per-iteration optimization objective; see \Cref{sec:related works}. 
Examples include SPPO~\citep{wu2025selfplay}, INPO~\citep{zhang2025iterative}, and practical implementations of the prox method~\citep{liu2026comal}. 
Concurrent and independent work~\citep{wu2026multiplayer} also implements learning dynamics with LLMs in a multi-player setting, though with a different utility function, as discussed in \Cref{sec:related works}. We leave it as future work to test the practical performance of U-alignment and the corresponding preference-alignment methods. 

Our framework also assumes that, given a set of responses, either the user or an auxiliary selector can choose the most preferred response. 
In practice, users may not want to inspect a large set of responses, while automated selectors may introduce errors. 
Thus, our results guarantee diversity and quality in the candidate response set generated by the policy, but the quality of the final selected response also depends on the accuracy of the selection process. 
Understanding how to design reliable and scalable selection mechanisms for U-alignment is an important direction for future work, especially for LLM alignment and personalization.

\bibliographystyle{plainnat}
\bibliography{ref, references}
\appendix

\tableofcontents

\section{Missing Details in the Preliminaries}\label{app:pre}
For $m, n \in \mathbb{Z}_+^{\+Y}$, we say $m \le n$ if $m(y) \le n(y)$ for all $y \in \+Y$. For $z \in \+Y$, we denote $e_z$ the unit vector with $e_z(z)=1$ and $e_z(y) = 0$ for $y\ne z$.

We introduce the definition of monotonicity and submodularity with diminishing returns (DR-submodular) below.
\begin{definition}[DR-submodularity on multisets]
Let $F:\mathbb{Z}_+^{\+Y}\to\mathbb{R}$. We say $F$ is \emph{DR-submodular} if for all $m\le n$
and all $z\in Y$,
\[
F(m+e_z)-F(m)\ \ge\ F(n+e_z)-F(n).
\]
We say $F$ is \emph{monotone} if $m\le n$ implies $F(m)\le F(n)$, and \emph{normalized} if $F(0)=0$.
\end{definition}

\subsection{Proof of Mixture of PL Preferences Satisfying \Cref{dfn:general preferences}}
The first three properties hold immediately by definition. We prove that the function $F_{x,y}: \-Z_{+}^{\+Y} \rightarrow [0,1]$ defined by
\[F_{x,y}(m):=\-P_{\+D}[S(m) \succcurlyeq y \mid x]\]
is normalized, monotone, and DR-submodular.

It is clear that $F_{x,y}(0) = 0$, hence the function is normalized.

Fix user $i$, its reward function $r_i$, and prompt $x$. Define weights $w_i(x,z) = \exp(r_i(x,z))$. Then for any opponent response $y \in \+Y$, we have 
\[
F_{x,y}^{(i)}(m) := \-P_{i}[S(m) \succcurlyeq y \mid x] =\frac{\sum_{z \in \+Y}m(z) \cdot w_i(x,z)}{w_i(x,y)+ \sum_{z \in \+Y}m(z) \cdot w_i(x,z)} = g\InParentheses{\sum_{z \in \+Y}m(z) \cdot w_i(x,z)},
\]
where $g(s) = \frac{s}{a+s}$ with $a = w_i(x,y)$. Since the map $g$ is increasing and concave over $\-R_{+}$ and the function $m \rightarrow \sum_{z \in \+Y}m(z) \cdot w_i(x,z)$ is increasing, we have $F_{x,y}^{(i)}$ is monotone and DR-submodular. Taking expectation over $i \sim \+D$ concludes that $F_{x,y}$ is monotone and DR-submodular. This concludes the proof.
\subsection{Proof of \Cref{prop:preference}}

\begin{proof}
\textbf{For the antisymmetry property}, let us fix $\succ_i$. When we draw $S\sim\pi(\cdot\mid x)$, $S'\sim\pi'(\cdot\mid x)$, we have the point-wise identity
\[
\mathbf{1}[S\succcurlyeq_{i,x}S']+\mathbf{1}[S'\succ_{i,x}S]=1.
\]
Taking the expectation over $\succ_i \sim \+D$ and the draws of $S,S'$ yields the antisymmetry property.

\textbf{For the second property}, let us fix any single-output policy $\pi$, any ranking $\succ_i$,  and any prompt $x$. Draw $k+1$ i.i.d. samples $Z_1,\dots,Z_k,Z_{k+1}\sim \pi(\cdot\mid x)$. Let us also sample a permutation $\sigma$ over the index set $[k+1]$ uniformly at random and define $Y_i = Z_{\sigma(i)}$ for $i \in [k+1]$. It is clear that $\{Y_i\}_{i \in [k+1]}$ are also $k+1$ i.i.d. samples from $\pi(\cdot \mid x)$. Now fix $Z =\{Z_i\}_{i \in [k+1]}$, let $y = Y_{k+1}$ and $S = \{Y_1, \ldots, Y_k\}$. Over the randomness of the permutation $\sigma$, we have
\begin{align*}
&\mathbb{E}\big[\mathbf{1}[y\succ_{i,x} S ]\big]
= \frac{1}{k+1} \sum_{j=1}^{k+1} \mathbf{1}[Z_j\succ Z \setminus\{Z_j\}]
\le \frac{1}{k+1}, \\
&\mathbb{E}\big[\mathbf{1}[S\succcurlyeq_{i,x} y]\big] = 1- \mathbb{E}\big[\mathbf{1}[y\succ_{i,x} S ]\big] \ge 1 - \frac{1}{k+1}.
\end{align*}
Since the above holds for all realizations of $Z$ and all $\succ_i$, we can take the expectation over Z and $\succ_i\sim \+D$ to
conclude
\[
\-P_{\+D}[\pi^{\otimes k}\succcurlyeq \pi\mid x]\ge1-\frac{1}{k+1}.
\]

\textbf{For the third property}, we note that the maximal element in $S_1 \cup S_2$ either lies in $S_1$ or lies in $S_2$. Therefore, we have
\[
\mathbf{1}[S_1 + S_2 \succ_i S \mid x] \le \mathbf{1}[S_1\succ_i S \mid x] + \mathbf{1}[S_2 \succ_i S \mid x]
\]
Taking expectation over $\succ_i \sim \+D$ completes the proof.

\textbf{For the fourth property}, fix user $i$, prompt $x$, and the opponent response $y$. Let the set of responses that is preferred over $y$ given prompt $x$ be
\[
U_{i,x}(y):= \{z \in \+Y: z \succcurlyeq_{i,x} y\}.
\]
Then for user $i$, the win rate function is an indicator function:
\[
F_{x,y}^{(i)}(m) = \textbf{1}\InBrackets{\sum_{z \in U_{i,x}(y)}m(z) \ge 1}.
\]
Clearly $F_{x,y}^{(i)}(0) = 0$ and $F_{x,y}^{(i)}(\cdot)$ is increasing. Moreover, its marginal is 
\begin{align*}
    F_{x,y}^{(i)}(m + e_z) - F_{x,y}^{(i)}(m) = \begin{cases}
        1 &\text{if } z \in  U_{i,x}(y) \text{ and } \sum_{z\in U_{i,x}(y)} m(z) = 0 \\
        0 &\text{otherwise}
    \end{cases}
\end{align*}
It is clear that as $m$ increases, the marginal can only decrease. So $F_{x,y}^{(i)}(\cdot)$ is DR-submodular. Taking expectation over $i \sim \+D$ concludes that $F_{x,y}$ is monotone and DR-submodular. This concludes the proof.
\end{proof}
\section{Missing Proofs in \Cref{sec:multi-output}}\label{app:multi-ouput}

\subsection{Proof of \Cref{prop:lower bound PL}}
\begin{proof}
    Consider a simple case where $\+X = \{x\}$ and $\+Y = \{y_1, y_2, \ldots, y_{k+1}\}$. Consider a Plackett-Luce preference with a uniform reward function $r_i(x,y) = 1$ for all $y \in \+Y$. For any (multi)-set $S$ of size $k$, we have $\-P[S\succcurlyeq y_1\mid x] = \frac{k}{k+1}$. Thus for any policy $\pi: \+X \rightarrow \Delta(\+Y^k)$, we have $\-P[\pi \succcurlyeq y_1] = \-E_{S \sim \pi(\cdot \mid x)}[P[S\succcurlyeq y_1\mid x]] = \frac{k}{k+1}$. This completes the proof.
\end{proof}

\subsection{Proof of \Cref{prop:lower bound rankings}}
\begin{proof}
    Let $\+X = \{x\}$ and $\+Y = \{y_1, \ldots, y_{m}\}$. We consider the population preference $\-P_{\+D}$ generated from the uniform distribution over all the permutations of $m$ responses, i.e., $m!$ rankings $\{\succ_i\}_{i \in [m!]}$. Consider any multiset $S \in \+Y^k$ and $y \in \+Y$, we have
    \[
    \-P_{\+D}\InBrackets{S \succcurlyeq y\mid x} \le \mathbf{1}\InBrackets{y \in S} + \mathbf{1}\InBrackets{y \notin S} \cdot \frac{k}{k+1}=\mathbf{1}\InBrackets{y \in S}\cdot \frac{1}{k+1} + \frac{k}{k+1},
    \]
    where the inequality holds since (1) $y \in S \implies S \succcurlyeq_i y$ for all $i$; (2) $y \notin S \implies \-P_{\+D}[y \succ S] \ge \frac{1}{k+1}$ as the win rate only depends on the relative ordering between $\{y\} \cup S$ and the population is the uniform distribution over all permutations of $m$ responses.
    
    Now let us fix any policy $\pi: \+X \rightarrow \Delta(\+Y^k)$. Since the policy $\pi$ samples $k$ responses, we have
    \[
    \sum_{j=1}^{m} \Pr_{S \sim \pi(\cdot \mid x)}[y_j \in S] \le k.
    \]
    Therefore, there exists $j \in [m]$ such that the response $y_j$ satisfies
    \[
    \Pr_{S \sim \pi(\cdot \mid x)}[y_j \in S] \le \frac{k}{m}.
    \]
    Consequently, we have
    \begin{align*}
       \-P_{\+D}[\pi \succcurlyeq y_j \mid x] \le \-E_{S \sim \pi(\cdot\mid x)}\InBrackets{ \mathbf{1}\InBrackets{y_j \in S}\cdot \frac{1}{k+1} + \frac{k}{k+1}} \le \frac{k}{m\cdot(k+1)} + \frac{k}{k+1}.
    \end{align*}
    This completes the proof.
\end{proof}

\section{Missing Proofs in \Cref{sec:single-output}}\label{app:single-ouput}
\subsection{Proof of \Cref{prop:limit of NLHF}}
\begin{proof}[Proof of \Cref{prop:limit of NLHF}]
Let $\+Y=\{y_1,y_2\}$ and fix any $\varepsilon\in(0,\tfrac12)$. Consider the population distribution $\+D$ over two rankings:
\[
\+D=
\begin{cases}
y_1 \succ y_2 & \text{with probability } \tfrac12+\varepsilon,\\
y_2 \succ y_1 & \text{with probability } \tfrac12-\varepsilon.
\end{cases}
\]
Then $\-P_{\+D}[y_1\succcurlyeq y_2]=\tfrac12+\varepsilon$.

It is clear that the unique Nash equilibrium policy under $\min_{\pi_1}\max_{\pi_2}\-P_{\+D}[\pi_1 \succcurlyeq \pi_2]$ is the deterministic policy $\pi_{\mathrm{NLHF}}$ that outputs $y_1$. Since $\pi_{\mathrm{NLHF}}$ is deterministic, test-time scaling does not change the output. Therefore, for any $k\ge 1$,
\[
\min_{y\in\+Y}\; \-P_{\+D}[\pi_{\mathrm{NLHF}}^{\otimes k}\succcurlyeq y]
=
\-P_{\+D}[y_1\succcurlyeq y_2]
=
\tfrac12+\varepsilon.
\]
This completes the proof.
\end{proof}

\subsection{On Multi-Player Nash Equilibrium and Optimal U-Aligned Policy}\label{sec:uniqueness}

In this section, we make some remarks on the relationship between multi-player Nash equilibria and U-aligned policies with 
\((k,\frac{k}{k+1})\)-robust alignment.

Recall that we have shown that any \emph{symmetric} \((k+1)\)-player Nash equilibrium achieves 
\((k,\frac{k}{k+1})\)-robust alignment (\Cref{theorem:single-output good}). 
The converse, however, does not hold: a U-aligned policy with 
\((k,\frac{k}{k+1})\)-robust alignment need not be a Nash equilibrium of the \((k+1)\)-player alignment game.

The reason is that the guarantee \(\frac{k}{k+1}\), while tight in the worst case 
(\Cref{prop:lower bound rankings}), can be strictly suboptimal on particular instances. 
As a simple illustrative example, consider a ranking setting with two alternatives, \(A\) and \(B\), where a 
\(\frac{k}{k+1}\) fraction of users prefer \(A\), while the remaining \(\frac{1}{k+1}\) fraction prefer \(B\). 
Then the deterministic policy that always outputs \(A\) achieves 
\((k,\frac{k}{k+1})\)-robust alignment. However, it is not a Nash equilibrium of the \((k+1)\)-player alignment game. 
Indeed, if all players output \(A\), then each player receives utility \(0\), since \(A\) does not strictly beat \(A\). 
By deviating to the deterministic policy that outputs \(B\), a player obtains utility \(\frac{1}{k+1}>0\), since this is precisely the fraction of users who prefer \(B\) to \(A\).

In this example, the unique symmetric Nash equilibrium outputs \(A\) with probability
\(
\frac{k^{1/k}}{1+k^{1/k}}
\)
and outputs \(B\) with probability
\(
\frac{1}{1+k^{1/k}}.
\)
Thus, although every symmetric Nash equilibrium yields the worst-case robust-alignment guarantee, policies achieving that guarantee need not coincide with Nash equilibria. The example also illustrates that the family of U-aligned policies satisfying the 
\((k,\frac{k}{k+1})\)-robust-alignment guarantee need not be unique.

On the other hand, the multi-player alignment game may itself have multiple Nash equilibria, especially asymmetric ones. 
For example, consider a ranking setting with two alternatives, \(A\) and \(B\), in which all users prefer \(A\). 
For \(k=2\), the corresponding \(3\)-player alignment game has, among others, the deterministic Nash equilibria 
\((A,A,A)\) and \((A,A,B)\), where the entries denote the three players' deterministic policies. 
The profile \((A,A,B)\) is an asymmetric Nash equilibrium: the player outputting \(B\) receives utility \(0\), but no deviation can give that player positive utility, since at least one other player already outputs \(A\); similarly, the players outputting \(A\) cannot profitably deviate. 
This shows that Nash equilibria of the multi-player alignment game are not unique in general and may include asymmetric equilibria.

\subsection{Proof of \Cref{thm:fixed scaling}}
We first establish a useful lemma based on (4) in \Cref{dfn:general preferences}.
\begin{lemma}[DR-submodularity $\Rightarrow$ concavity in sample size]\label{lemma:submodular}
Assume $F:\mathbb{Z}_+^{Y}\to\mathbb{R}$ is normalized, monotone, and DR-submodular.
Let $M_\ell=\sum_{t=1}^\ell e_{Z_t}$ for $Z_t\stackrel{i.i.d.}{\sim}\pi$ and define
$G(\ell)=\mathbb{E}[F(M_\ell)]$.
Then:
\begin{enumerate}
\item $G(\ell)$ is nondecreasing in $\ell$.
\item $G(\ell)$ is \emph{discrete concave}: $G(\ell+2)-G(\ell+1)\le G(\ell+1)-G(\ell)$ for all $\ell\ge 0$.
\item Consequently, for any integer $k\ge 1$ and all $0\le \ell\le k$,
\[
G(\ell)\ \ge\ \frac{\ell}{k}\,G(k).
\]
\end{enumerate}
\end{lemma}

\begin{proof}
(1) By monotonicity, $F(M_{\ell+1})=F(M_\ell+e_{Z_{\ell+1}})\ge F(M_\ell)$, so taking expectation gives $G(\ell+1)\ge G(\ell)$.

(2) Let $\Delta(z\mid m):=F(m+e_z)-F(m)$ be the marginal gain of adding one copy of $z$.
Then
\[
G(\ell+1)-G(\ell)
=\mathbb{E}\big[F(M_\ell+e_{Z_{\ell+1}})-F(M_\ell)\big]
=\mathbb{E}\big[\Delta(Z_{\ell+1}\mid M_\ell)\big].
\]
DR-submodularity guarantees that for every fixed $z$, the function $m\mapsto \Delta(z\mid m)$ is nonincreasing w.r.t.\ the partial order.
Under the natural coupling $M_{\ell+1}=M_\ell+e_{Z_{\ell+1}}\ge M_\ell$, and since $\Delta(z\mid\cdot)$ is nonincreasing,
\[
\mathbb{E}\big[\Delta(Z_{\ell+2}\mid M_{\ell+1})\big]\ \le\ \mathbb{E}\big[\Delta(Z_{\ell+2}\mid M_\ell)\big].
\]
Using independence of $Z_{\ell+2}$ from $(M_\ell,M_{\ell+1})$, the RHS equals $\mathbb{E}[\Delta(Z_{\ell+1}\mid M_\ell)]$.
Hence $G(\ell+2)-G(\ell+1)\le G(\ell+1)-G(\ell)$, proving discrete concavity.

(3) Discrete concavity implies the forward differences $d_\ell:=G(\ell)-G(\ell-1)$ are nonincreasing in $\ell$.
Thus for $1\le \ell\le k$,
\[
\frac{G(\ell)}{\ell}=\frac{1}{\ell}\sum_{t=1}^\ell d_t \ \ge\ \frac{1}{k}\sum_{t=1}^k d_t=\frac{G(k)}{k},
\]
which rearranges to $G(\ell)\ge \frac{\ell}{k}G(k)$. This completes the proof.
\end{proof}

\begin{proof}[Proof of \Cref{thm:fixed scaling}]
    We know $\pi_k$ achieves $(k, \nicefrac{k}{k+1})$ robust alignment. Fix any prompt $x$ and opponent response $y$. Recall the function $F_{x,y}: \-Z_{+}^{\+Y} \rightarrow \-R$ defined as $F_{x,y}(m) = \-P_{\+D}[S(m) \succcurlyeq y]$ is monotone, normalized, and DR-submodular (\Cref{dfn:general preferences} (4)). Then using the notations in \Cref{lemma:submodular}, we have
    \begin{align*}
        G(k) = \-E_{S \sim \pi_k^{\otimes k}}\InBrackets{F_{x,y}[m_S]} = \-P_{\+D}[\pi_k^{\otimes k} \succcurlyeq y \mid x] \ge \frac{k}{k+1}.
    \end{align*}
     By \Cref{lemma:submodular}, we have for any $0 \le \ell \le k$, 
    \begin{align*}
        \-P_{\+D}[\pi_k^{\otimes \ell} \succcurlyeq y\mid x] = G(\ell) \ge \frac{\ell}{k} G(k) \ge \frac{\ell}{k+1}.
    \end{align*}
    For $k \ge \ell$, monotonicity of $F_{x,y}(\cdot)$ implies $\-P_{\+D}[\pi_k^{\otimes \ell} \succcurlyeq y \mid x] = G(\ell)\ge G(k) = \nicefrac{k}{k+1}$. This completes the proof.
\end{proof}

\subsection{Proof of \Cref{proposition:fixed point}}
\begin{proof}
It is not hard to see that the map
$\pi\mapsto \nabla_{\pi_1}u_1(\pi,\pi^{\otimes k})$ is continuous. The projection operator
$\Pi_{\Delta(\+Y)}$ is continuous, and $\Delta(\+Y)$ is compact and convex. Therefore, $F_\eta$
is a continuous map from $\Delta(\+Y)$ to itself. By Brouwer's fixed-point theorem, $F_\eta$ has
at least one fixed point.

Now let $\pi$ be any fixed point of $F_\eta$, and write
\[
g^\pi := \nabla_{\pi_1} u_1(\pi,\pi^{\otimes k}).
\]
The fixed-point condition $\pi = \Pi_{\Delta(\+Y)}(\pi+\eta \cdot g^\pi)$ implies, by the first-order
optimality condition for Euclidean projection onto a convex set, that
\[
\langle g^\pi,\ \pi'-\pi\rangle \le 0,\qquad \forall \pi'\in\Delta(\+Y).
\]
Equivalently,
\[
\langle g^\pi,\ \pi'\rangle \le \langle g^\pi,\ \pi\rangle,\qquad \forall \pi'\in\Delta(\+Y).
\]
Since $u_1(\cdot,\pi^{\otimes k})$ is linear in its first argument, the gradient $g^\pi$
represents the linear functional defining player~$1$'s payoff against $\pi^{\otimes k}$, and the inequality above is exactly the best-response condition:
\[
u_1(\pi',\pi^{\otimes k}) \le u_1(\pi,\pi^{\otimes k}),\qquad \forall \pi'\in\Delta(\+Y).
\]
By symmetry, the same holds for every player. Hence $(\pi,\ldots,\pi)$ is a symmetric Nash
equilibrium of the multi-player alignment game.
\end{proof}

\subsection{Proof of \Cref{prop:self-play}}
\begin{proof}
By the definition of regret, for any $\pi\in\Delta(\+Y)$,
\[
\frac{1}{T}\sum_{t=1}^T u(\pi,(\pi^t)^{\otimes k})
\le
\frac{1}{T}\sum_{t=1}^T u(\pi^t,(\pi^t)^{\otimes k})
+ \frac{\reg^T}{T}.
\]
Since $\sigma^T$ uniformly samples from $\{(\pi^t)^{\otimes k}\}_{t\in[T]}$, the left-hand
side equals $u(\pi,\sigma^T)$. Moreover, by \Cref{dfn:general preferences}, for each $t$ we have
$u(\pi^t,(\pi^t)^{\otimes k}) \le \frac{1}{k+1}$, and hence the first term on the right-hand
side is also at most $\frac{1}{k+1}$. Therefore,
\[
\max_{\pi\in\Delta(\+Y)} u(\pi,\sigma^T)
\le
\frac{1}{k+1} + \frac{\reg^T}{T}.
\]

Applying \Cref{dfn:general preferences} once more, we obtain
\[
\min_{\pi\in\Delta(\+Y)} \-P_{\+D}[\sigma^T \succcurlyeq \pi]
\ge
\frac{k}{k+1} - \frac{\reg^T}{T},
\]
which proves the claim.
\end{proof}

\section{Missing Proofs for \Cref{sec:multi-output-opponents}}
\label{app:multi-output-opponents}

\subsection{Proof of \Cref{thm:arbitrary-multi-output-opponent}}
\begin{proof}[Proof of \Cref{thm:arbitrary-multi-output-opponent}]
Fix $k,\ell\ge 1$, and let $\pi_k$ be a symmetric Nash equilibrium policy of the
$(k+1)$-player alignment game in \Cref{dfn:multi-player alignment game}. By
\Cref{theorem:single-output good}, for every prompt $x\in\+X$ and every
single-output policy $\pi:\+X\to\Delta(\+Y)$,
\[
    \-P_{\+D}\InBrackets{\pi_k^{\otimes k}\succcurlyeq \pi\mid x}
    \ge
    \frac{k}{k+1}.
\]
By antisymmetry in \Cref{dfn:general preferences}, this implies that for every
response $y\in\+Y$,
\[
    \-P_{\+D}\InBrackets{y\succ \pi_k^{\otimes k}\mid x}
    \le
    \frac{1}{k+1}.
\]
Now fix any multiset $S=(y^{(1)},\ldots,y^{(\ell)})\in\+Y^\ell$. By the
subadditivity property in \Cref{dfn:general preferences},
\[
    \-P_{\+D}\InBrackets{S\succ \pi_k^{\otimes k}\mid x}
    \le
    \sum_{i=1}^{\ell}
    \-P_{\+D}\InBrackets{y^{(i)}\succ \pi_k^{\otimes k}\mid x}
    \le
    \frac{\ell}{k+1}.
\]
Taking the maximum over deterministic multisets $S$ and then over distributions
on $\+Y^\ell$ gives
\[
    \max_{\pi':\+X\to\Delta(\+Y^\ell)}
    \-P_{\+D}\InBrackets{\pi'\succ \pi_k^{\otimes k}\mid x}
    \le
    \min\left\{1,\frac{\ell}{k+1}\right\}.
\]
Applying antisymmetry again yields
\[
    \min_{\pi':\+X\to\Delta(\+Y^\ell)}
    \-P_{\+D}\InBrackets{\pi_k^{\otimes k}\succcurlyeq \pi'\mid x}
    \ge
    \max\left\{0,\frac{k+1-\ell}{k+1}\right\}.
\]
Since the bound holds for every prompt $x$, the theorem follows.
\end{proof}

\subsection{Examples of Random Utility Preferences}
\label{app:rum-examples}
We verify the examples in \Cref{remark:rum-examples}. Throughout, we fix the prompt $x$.

\paragraph{Mixtures of PL preferences} Fix a PL type $i$ with reward function $r_i$, and let every copy $j$ of a response $y_j$ receive the utility $V_j=r_i(x,y_j)+\xi_j$, where the $\xi_j$ are i.i.d.\ standard Gumbel random variables, i.e., $\Pr[\xi_j\le g]=\exp(-e^{-g})$ for $g\in\R$. This is the standard random utility construction of Luce's choice probabilities~\citep{train2009discrete}. Let $w(y):=\exp(r_i(x,y))$ and $W(S):=\sum_{y\in S}w(y)$, where copies are counted with multiplicity. Since $e^{-\xi_j}$ follows the exponential distribution with rate $1$, the random variables
\[
    E_j:=e^{-V_j}=w(y_j)^{-1}e^{-\xi_j}
\]
are independent exponential random variables with rates $w(y_j)$, and maximizing the utility is equivalent to minimizing $E_j$. The minima of $\{E_j\}_{j\in S}$ and $\{E_j\}_{j\in T}$ are independent exponential random variables with rates $W(S)$ and $W(T)$, respectively. Therefore,
\[
    \Pr\InBrackets{\max_{j\in S}V_j\ge\max_{j\in T}V_j}
    =
    \int_0^\infty W(S)e^{-(W(S)+W(T))t}\,dt
    =
    \frac{W(S)}{W(S)+W(T)},
\]
which is exactly the PL preference in \Cref{sec:prelim}, with every copy counted. Ties occur with probability zero, so $\succcurlyeq$ and $\succ$ coincide. Taking the type $\omega=i\sim\+D$ recovers the mixtures of PL preferences.

\paragraph{Thurstone--Mosteller preferences} Let $Q_{\omega,y}=\N(\mu_\omega(y),\sigma_\omega(y)^2)$ with $\sigma_\omega(y)>0$, i.e., $V_j=\mu_\omega(y_j)+\sigma_\omega(y_j)\xi_j$ with i.i.d.\ standard Gaussian $\xi_j$. This is the Gaussian (probit) random utility model~\citep{train2009discrete,vojnovic2017parameter}. For a copy of $a$ with utility $V_a$ and a copy of $b$ with utility $V_b$, conditional on $\omega$, the difference $V_a-V_b$ is Gaussian with mean $\mu_\omega(a)-\mu_\omega(b)$ and variance $\sigma_\omega(a)^2+\sigma_\omega(b)^2$, so
\[
    \Pr\InBrackets{V_a>V_b\mid\omega}
    =
    \Phi_{\N}\InParentheses{\frac{\mu_\omega(a)-\mu_\omega(b)}{\sqrt{\sigma_\omega(a)^2+\sigma_\omega(b)^2}}},
\]
where $\Phi_{\N}$ is the cumulative distribution function of the standard Gaussian distribution. The common-variance case $\sigma_\omega(y)\equiv\sigma_\omega$ is the classical Thurstone--Mosteller model. An arbitrary mixture over types $\omega$, and hence over means and variances, is allowed; the resulting pairwise preference is the average of the above expression over $\omega\sim\+D$, which need not itself be of this form. Our results apply to multisets of any sizes $k$ and $\ell$, even though their win rates do not admit closed forms. Ties occur with probability zero.

\paragraph{Populations of rankings and correlated utilities} Let the type $\omega=(\omega_y)_{y\in\+Y}\in\R^{\+Y}$ be a random vector with an arbitrary joint distribution $\+D$, possibly with atoms and arbitrary correlations, e.g., a Gaussian vector $\omega\sim\N(\mu,\Sigma)$ with an arbitrary positive semidefinite covariance matrix $\Sigma$. Let $Q_{\omega,y}$ be the point mass at $\omega_y$. Then all copies of $y$ share the same utility $\omega_y$, and the conditional independence in \Cref{dfn:rum-preference} holds trivially because the utilities are deterministic given $\omega$. For a population of rankings, draw a ranking $\succ_i$ from $\+D$ and set $\omega_y:=-|\{z\in\+Y:z\succ_{i,x}y\}|$, i.e., minus the number of responses ranked above $y$. Then the copies with the largest utility in a multiset are exactly the copies of its most preferred response, so $\max_{j\in S}V_j\ge\max_{j\in T}V_j$ if and only if $\boldsymbol{1}[S\succcurlyeq_{i,x}T]=1$, and $\max_{j\in S}V_j>\max_{j\in T}V_j$ if and only if $\boldsymbol{1}[S\succ_{i,x}T]=1$. This recovers exactly the population of rankings in \Cref{sec:prelim}, including its weak-tie convention.

\subsection{Proof of \Cref{prop:product-game-concave-convex}}
\begin{proof}[Proof of \Cref{prop:product-game-concave-convex}]
Fix a prompt $x\in\+X$. Since averaging over the type $\omega\sim\+D$ preserves concavity in $p$ and convexity in $q$, it suffices to prove the claims for a fixed type $\omega$. For every response $y$ and $v\in\R$, let $Q_{\omega,y}(v):=\Pr_{V\sim Q_{\omega,y}}[V\le v]$ and $Q_{\omega,y}(v^-):=\Pr_{V\sim Q_{\omega,y}}[V<v]$. Conditional on $\omega$, the $k+\ell$ sampled copies, each consisting of a response and its utility, are mutually independent: a copy of the model is drawn as $Y\sim p$ and $V\sim Q_{\omega,Y}$, and a copy of the opponent is drawn as $Y\sim q$ and $V\sim Q_{\omega,Y}$. In particular, the utility $V$ of a single sample from $p$ satisfies
\[
    \Pr\InBrackets{V<v\mid\omega}=\sum_{y\in\+Y}p(y)\,Q_{\omega,y}(v^-),
\]
which is affine in $p$, and similarly for $q$.

\paragraph{Weak-tie comparison} Fix an opponent multiset $T$, and let $b$ be the best utility among its copies. The model weakly beats the opponent if and only if at least one of its $k$ samples has utility at least $b$. Therefore,
\[
    \Pr_{S\sim p^{\otimes k}}\left[\max_{j\in S}V_j\ge b\;\middle|\;\omega,b\right]
    =
    1-\InParentheses{\sum_{y\in\+Y}p(y)\,Q_{\omega,y}(b^-)}^k.
\]
The map $p\mapsto \sum_{y\in\+Y}p(y)Q_{\omega,y}(b^-)$ is affine with values in $[0,1]$, and $z\mapsto 1-z^k$ is concave on $[0,1]$. Hence the win rate is concave in $p$. Since the distribution of $b$ does not depend on $p$, averaging over $b$ and over $T\sim q^{\otimes \ell}$ preserves concavity in $p$.

Similarly, fix a model multiset $S$, and let $a$ be the best utility among its copies. The model weakly beats the opponent if and only if none of the opponent's $\ell$ samples has utility larger than $a$. Hence
\[
    \Pr_{T\sim q^{\otimes \ell}}\left[\max_{j\in T}V_j\le a\;\middle|\;\omega,a\right]
    =
    \InParentheses{\sum_{y\in\+Y}q(y)\,Q_{\omega,y}(a)}^\ell.
\]
The map $q\mapsto \sum_{y\in\+Y}q(y)Q_{\omega,y}(a)$ is affine with values in $[0,1]$, and $z\mapsto z^\ell$ is convex on $[0,1]$. Therefore the win rate is convex in $q$. Since the distribution of $a$ does not depend on $q$, averaging over $a$ and over $S\sim p^{\otimes k}$ preserves convexity in $q$.

\paragraph{Randomized tie-breaking} We implement the uniform tie-breaking by attaching to every sampled copy $j$ an independent priority $R_j\sim\Unif([0,1])$ and selecting the copy with the lexicographically largest augmented utility $(V_j,R_j)$. With probability one, the augmented utilities are distinct, and the selected copy is uniformly distributed among the copies with the maximal utility. For $(v,u)\in\R\times[0,1]$, write $(V,R)\le_{\mathrm{lex}}(v,u)$ if $V<v$, or $V=v$ and $R\le u$. Then the augmented utility $(V,R)$ of a single sample from $p$ satisfies
\[
    H_{\omega,p}(v,u)
    :=
    \Pr\InBrackets{(V,R)\le_{\mathrm{lex}}(v,u)\mid\omega}
    =
    \sum_{y\in\+Y}p(y)\InBrackets{Q_{\omega,y}(v^-)+u\InParentheses{Q_{\omega,y}(v)-Q_{\omega,y}(v^-)}},
\]
which is affine in $p$ with values in $[0,1]$. Let $B$ be the best augmented utility among the opponent's samples. Conditional on $\omega$ and $B$, the model wins if and only if not all of its $k$ augmented utilities fall below $B$ (equality occurs with probability zero), so its conditional win rate is $1-H_{\omega,p}(B)^k$, which is concave in $p$. Similarly, let $A$ be the best augmented utility among the model's samples. Conditional on $\omega$ and $A$, the model wins if and only if all $\ell$ augmented utilities of the opponent fall below $A$, so its conditional win rate is $H_{\omega,q}(A)^\ell$, which is convex in $q$. Since the distribution of $B$ does not depend on $p$ and the distribution of $A$ does not depend on $q$, averaging over $B$ and $A$ preserves concavity in $p$ and convexity in $q$, respectively.

Finally, averaging over the type $\omega\sim\+D$ preserves both concavity and convexity. This completes the proof.
\end{proof}

For a population of rankings, $1-\sum_{y\in\+Y}p(y)Q_{\omega,y}(b^-)$ is the probability that a sample from $p$ is weakly preferred to the opponent's best response. For a PL preference with weights $w(y)=\exp(r_i(x,y))$, the Gumbel construction in \Cref{app:rum-examples} gives $Q_{i,y}(v)=\exp(-w(y)e^{-v})$, and the substitution $t=e^{-b}$ turns the weak-tie argument into the Laplace-transform identity
\[
    \-E_{S\sim p^{\otimes k}}
    \left[\frac{W(S)}{W(S)+W(T)}\right]
    =
    1-
    W(T)\int_0^\infty e^{-W(T)t}
    \left(
        \sum_{y\in\+Y}p(y)e^{-t w(y)}
    \right)^k dt,
\]
where $W(S):=\sum_{y\in S}w(y)$.

\subsection{Proof of \Cref{thm:sharp-product-opponent}}
\begin{proof}[Proof of \Cref{thm:sharp-product-opponent}]
Fix a prompt $x\in\+X$. By \Cref{prop:product-game-concave-convex},
$\Phi_{x,k,\ell}$ is concave in $p$ and convex in $q$ for random utility
preferences. Since $\Delta(\+Y)$ is compact and convex and
$\Phi_{x,k,\ell}$ is continuous, by Sion's minimax theorem~\citep{sion1958general}, we have
\[
    \max_{p\in\Delta(\+Y)}\min_{q\in\Delta(\+Y)}
    \Phi_{x,k,\ell}(p,q)
    =
    \min_{q\in\Delta(\+Y)}\max_{p\in\Delta(\+Y)}
    \Phi_{x,k,\ell}(p,q).
\]
For every fixed $q$, choosing $p = q$ gives
\[
    \max_{p\in\Delta(\+Y)}\Phi_{x,k,\ell}(p,q)
    \ge
    \Phi_{x,k,\ell}(q,q),
\]
and therefore
\[
    \max_{p}\min_{q}\Phi_{x,k,\ell}(p,q)
    \ge
    \min_{q}\Phi_{x,k,\ell}(q,q).
\]

Now we lower bound the diagonal value $\Phi_{x,k,\ell}(q,q)$ for any $q$. We first fix a type $\omega$. Let us then draw
$Z_1,\ldots,Z_{k+\ell}\stackrel{\mathrm{i.i.d.}}{\sim}q$, assign the first $k$
samples to the model and the last $\ell$ samples to the opponent, and draw their utilities $V_j\sim Q_{\omega,Z_j}$ independently. Among the indices whose utilities are best among the
$k+\ell$ sampled copies, select one index uniformly at random. Since the pairs $(Z_j,V_j)$ are i.i.d., by
exchangeability, this selected best index lies among the model's first $k$
samples with probability $\nicefrac{k}{k+\ell}$. Whenever this happens, the
model weakly beats the opponent, and under randomized tie-breaking, the model wins exactly when this happens. Averaging over types $\omega\sim\+D$, we obtain
\[
    \Phi_{x,k,\ell}(q,q)
    \ge
    \frac{k}{k+\ell},
\]
with equality under randomized tie-breaking, or if ties occur with probability zero. Thus,
\[
    \max_{p}\min_{q}\Phi_{x,k,\ell}(p,q)
    \ge
    \frac{k}{k+\ell}.
\]

The above proofs fix one prompt $x \in \+X$. We can easily obtain a contextual policy by choosing for each prompt $x$, any maximizer
$p_x$ of the product-opponent game and defining
$\pi_{k,\ell}(\cdot\mid x):=p_x$.

It remains to prove tightness. Consider a single homogeneous PL preference in
which every response has the same weight, say $\exp(r(x,y))=1$ for all
$y\in\+Y$. Then every model multiset of size $k$ has total PL weight $k$, and
every opponent multiset of size $\ell$ has total PL weight $\ell$. Therefore,
for every $p,q\in\Delta(\+Y)$,
\[
    \-P\InBrackets{p^{\otimes k}\succcurlyeq q^{\otimes \ell}\mid x}
    =
    \frac{k}{k+\ell}.
\]
No distribution-free guarantee larger than $\nicefrac{k}{k+\ell}$ is possible,
even within the class of PL preferences.
\end{proof}

\subsection{Proof of \Cref{lemma:symmetric-saddle}}
\begin{proof}[Proof of \Cref{lemma:symmetric-saddle}]
The payoff $\Phi_{x,k,\ell}(p,q)=\-E_{S\sim p^{\otimes k},T\sim q^{\otimes\ell}}\InBrackets{K_x(S,T)}$ is a polynomial in $(p,q)$ and hence continuous. Since $\Delta(\+Y)$ is compact and convex, and $\Phi_{x,k,\ell}$ is concave in $p$ and convex in $q$, Sion's minimax theorem~\citep{sion1958general} gives
\[
    \max_{p\in\Delta(\+Y)}\min_{q\in\Delta(\+Y)}\Phi_{x,k,\ell}(p,q)
    =
    \min_{q\in\Delta(\+Y)}\max_{p\in\Delta(\+Y)}\Phi_{x,k,\ell}(p,q).
\]
For every $p$, we have $\min_{q}\Phi_{x,k,\ell}(p,q)\le\Phi_{x,k,\ell}(p,p)=c_{k,\ell}$, so the left-hand side is at most $c_{k,\ell}$. For every $q$, we have $\max_{p}\Phi_{x,k,\ell}(p,q)\ge\Phi_{x,k,\ell}(q,q)=c_{k,\ell}$, so the right-hand side is at least $c_{k,\ell}$. Hence both sides are equal to $c_{k,\ell}$.

Now let $r$ be maximin-optimal, so that $\Phi_{x,k,\ell}(r,q)\ge c_{k,\ell}$ for every $q$. Fix any $p\in\Delta(\+Y)$, and let $s_\lambda:=(1-\lambda)r+\lambda p$ for $\lambda\in(0,1)$. By the diagonal identity and concavity in the first argument,
\[
    c_{k,\ell}
    =
    \Phi_{x,k,\ell}(s_\lambda,s_\lambda)
    \ge
    (1-\lambda)\Phi_{x,k,\ell}(r,s_\lambda)+\lambda\Phi_{x,k,\ell}(p,s_\lambda)
    \ge
    (1-\lambda)c_{k,\ell}+\lambda\Phi_{x,k,\ell}(p,s_\lambda).
\]
Therefore, $\Phi_{x,k,\ell}(p,s_\lambda)\le c_{k,\ell}$. Letting $\lambda\downarrow 0$ and using continuity gives $\Phi_{x,k,\ell}(p,r)\le c_{k,\ell}$. Since $p$ is arbitrary, $\max_{p}\Phi_{x,k,\ell}(p,r)\le c_{k,\ell}=\min_{q}\max_{p}\Phi_{x,k,\ell}(p,q)$, so $r$ is also minimax-optimal. Together with $\Phi_{x,k,\ell}(r,q)\ge c_{k,\ell}$ for every $q$, this shows that $(r,r)$ is a saddle point.
\end{proof}

\subsection{Random Utility Preferences Satisfy \Cref{dfn:general preferences}}
\label{app:rum-property}
Let $\-P_{\+D}$ be a random utility preference (\Cref{dfn:rum-preference}), and fix a prompt $x\in\+X$.

\textbf{For the antisymmetry property}, fix a type $\omega$ and multisets $S,S'$. The event that $\max_{j\in S}V_j\ge\max_{j\in S'}V_j$ and the event that $\max_{j\in S'}V_j>\max_{j\in S}V_j$ are complementary. Taking the expectation over $\omega\sim\+D$ and the draws of $S\sim\pi(\cdot\mid x)$ and $S'\sim\pi'(\cdot\mid x)$ yields the antisymmetry property.

\textbf{For the second property}, fix any single-output policy $\pi$. The exchangeability argument in the proof of \Cref{thm:sharp-product-opponent} with $\ell=1$ and $q=\pi(\cdot\mid x)$ shows that
\[
    \-P_{\+D}\InBrackets{\pi^{\otimes k}\succcurlyeq\pi\mid x}
    =
    \Phi_{x,k,1}\InParentheses{\pi(\cdot\mid x),\pi(\cdot\mid x)}
    \ge
    1-\frac{1}{k+1}.
\]

\textbf{For the third property}, fix a type $\omega$ and multisets $S_1,S_2,S$, and draw the utilities of all copies in $S_1$, $S_2$, and $S$ jointly. Since the utility distribution of each copy depends only on $\omega$ and its response, this coupling has the correct distribution for each of the three comparisons. Since the maximal utility in $S_1+S_2$ is attained in $S_1$ or in $S_2$,
\[
    \boldsymbol{1}\InBrackets{\max_{j\in S_1+S_2}V_j>\max_{j\in S}V_j}
    \le
    \boldsymbol{1}\InBrackets{\max_{j\in S_1}V_j>\max_{j\in S}V_j}
    +
    \boldsymbol{1}\InBrackets{\max_{j\in S_2}V_j>\max_{j\in S}V_j}.
\]
Taking the expectation over the utilities and $\omega\sim\+D$ yields the subadditivity property.

\textbf{For the fourth property}, fix a type $\omega$ and the opponent response $y$, and let $b$ be the utility of its copy. Conditional on $\omega$ and $b$, the copies in $S(m)$ receive independent utilities, so
\[
    \Pr\left[\max_{j\in S(m)}V_j\ge b\;\middle|\;\omega,b\right]
    =
    1-\prod_{z\in\+Y}\theta_z^{m(z)},
    \qquad\text{where }\theta_z:=\Pr_{V\sim Q_{\omega,z}}[V<b]\in[0,1].
\]
This function of $m$ equals $0$ at $m=0$ and is monotone. Its marginal gain from adding one copy of $z$ is $(1-\theta_z)\prod_{z'\in\+Y}\theta_{z'}^{m(z')}$, which is nonincreasing in $m$, so it is DR-submodular. Taking the expectation over $b$ and $\omega\sim\+D$ preserves these properties. Hence $F_{x,y}$ is normalized, monotone, and DR-submodular. This completes the proof.

\section{A Practical Algorithm for U-Alignment}
\label{app:algorithm}
We present a practical algorithm product-opponent $\ell$-\ualign in \Cref{alg:comal-rebel}, by adapting the framework of COMAL~\citep{liu2026comal}. It follows the implementation of COMAL with the REBEL loss~\citep{gao2024rebel} in \citet{liu2026comal}, where we replace the win rate by the gradient $g_{x,k,\ell}$ of the product-opponent payoff, estimated from samples. 
COMAL is a double-loop algorithm where the inner loop repeated solved a KL-regularized two-player zero-sum game using mirror descent. The mirror descent step can be implemented by many reinforcement learning algorithms and we use REBEL~\citep{gao2024rebel} as an example. After each inner-loop optimization, COMAL resets the reference policy to be the current policy and starts another inner-loop optimization. Below we present more details on the implementation of the algorithm.

We fix a prompt $x\in\+X$ and $k,\ell\ge 1$. Recall from \Cref{sec:last-iterate} that $g_{x,k,\ell}(\mu)(y)=k\,\-E_{M\sim\mu^{\otimes(k-1)},T\sim\mu^{\otimes\ell}}[K_x(M+\{y\},T)]$ is $k$ times the win rate of $y$, together with $k-1$ responses from $\mu$, against $\ell$ responses from $\mu$. Given responses $y^{(1)},\ldots,y^{(B)}$ sampled from $\mu$, we estimate it by $\hat g(y^{(i)})$, which is $k$ times the empirical win rate of $y^{(i)}$, together with $k-1$ other sampled responses, against $\ell$ other sampled responses. When $k=\ell=1$, $\hat g(y^{(i)})$ is the empirical win rate of $y^{(i)}$ against $\mu$, as in COMAL. Given a reference policy $\pi_{\mathrm{ref}}$, regularization $\tau>0$, and step size $\eta>0$, the REBEL loss for the KL-regularized game, with $(y,y')$ drawn uniformly from the pairs of sampled responses, is
\begin{equation}
\label{eq:rebel-product}
\begin{aligned}
    \ell_{\mathrm{REBEL}}(\pi)
    :=
    \-E_{(y,y')}\Bigg[\bigg(
    &\frac{1}{\eta}\InParentheses{\log\frac{\pi(y)}{\mu(y)}-\log\frac{\pi(y')}{\mu(y')}}\\
    &-\InParentheses{\hat g(y)-\tau\log\frac{\mu(y)}{\pi_{\mathrm{ref}}(y)}-\hat g(y')+\tau\log\frac{\mu(y')}{\pi_{\mathrm{ref}}(y')}}\bigg)^{2}\Bigg].
\end{aligned}
\end{equation}

We leave testing the practical performance of \Cref{alg:comal-rebel} as future work.

\begin{algorithm}[ht]
\DontPrintSemicolon
\LinesNotNumbered
\caption{COMAL integrated with REBEL for product-opponent $\ell$-U-alignment}
\label{alg:comal-rebel}
\KwIn{Initial policy $\pi_{\mathrm{init}}$, numbers of responses $k,\ell\ge 1$, regularization $\{\tau_t>0\}$, step size $\{\eta_t>0\}$, number of outer iterations $T\ge 1$, number of inner iterations $\{N_t\ge 1\}$, number of samples $B\ge 1$, preference oracle $\mathbb{P}_{\+D}$.}
\KwOut{Optimized policy $\pi^{T}$}
Initialize $\pi^{1},\pi_{\mathrm{ref}}\leftarrow\pi_{\mathrm{init}}$\;
\For{$t=1,2,\ldots,T-1$}{
    Initialize $\mu^{1}\leftarrow\pi_{\mathrm{ref}}$\;
    \For{$n=1,2,\ldots,N_t-1$}{
        Generate responses $\{y^{(i)}\sim\mu^{n}\}_{i=1}^{B}$\;
        Query preference oracle $\mathbb{P}_{\+D}$ to estimate the gradient $\hat g^{n}(y^{(i)})\approx g_{x,k,\ell}(\mu^{n})(y^{(i)})$ for each $i\in[B]$\;
        Form dataset $\{(y^{(i)},y^{(j)},\hat g^{n}(y^{(i)}),\hat g^{n}(y^{(j)}))\}_{i,j\in[B]}$\;
        Compute $\mu^{n+1}=\argmin_{\pi}\ell_{\mathrm{REBEL}}(\pi)$ in \eqref{eq:rebel-product} with $\mu=\mu^{n}$, $\hat g=\hat g^{n}$, $\tau=\tau_t$, and $\eta=\eta_t$\;
    }
    $\pi^{t+1}\leftarrow\mu^{N_t}$\;
    $\pi_{\mathrm{ref}}\leftarrow\pi^{t+1}$\;
}
\Return{$\pi^{T}$}
\end{algorithm}


\end{document}